\documentclass[11pt]{amsart}

\usepackage{amsaddr}

\usepackage{amsthm}
\theoremstyle{definition}
\newtheorem{theorem}{Theorem}
\newtheorem{definition}{Definition}
\newtheorem{remark}{Remark}
\newtheorem{proposition}{Proposition}
\newtheorem{lemma}{Lemma}

\usepackage[parfill]{parskip}

\usepackage{graphicx}
\usepackage{epsfig}
\usepackage{amssymb}
\usepackage{natbib}

\long\def\acks#1{\vskip 0.3in\noindent{\large\bf Acknowledgments and Disclosure of Funding}\vskip 0.2in
\noindent #1}

\usepackage[font=footnotesize]{caption}
\usepackage[singlelinecheck=false,font=footnotesize]{subcaption}
\usepackage{comment}
\usepackage{xcolor}
\usepackage{algorithm}
\usepackage{algpseudocode}

\usepackage[hidelinks]{hyperref}
\hypersetup{
    colorlinks=true,
    linkcolor=gray,
    citecolor=gray,
    urlcolor=gray
}

\usepackage{amsmath}

\usepackage{amsthm}  

\usepackage[noabbrev, capitalize]{cleveref}
\newcommand{\thediag}{\partial \Omega}
\newcommand{\projdiag}{p_{\thediag}}
\newcommand{\KL}{\mathrm{KL}}
\newcommand{\FG}{\mathrm{FG}}
\newcommand{\SFG}{\mathrm{SFG}}

\newcommand{\pers}{\mathrm{pers}}
\newcommand{\RR}{\mathbb{R}}
\newcommand{\DD}{\mathcal{D}}
\newcommand{\spt}{\mathrm{supp}}
\newcommand{\entropicregparam}{\epsilon}

\newcommand{\ntv}{\texttt{Node2vec}}

\newcommand{\vrparameter}{\rho}

\usepackage{lastpage}

\begin{document}

\title[Topological Node2vec]{Topological Node2vec: Enhanced Graph Embedding via Persistent Homology}

\author{Yasuaki Hiraoka}
\address{\vspace{-0.5cm}Institute for the Advanced Study of Human Biology\\ Kyoto University Institute for Advanced Study\\ Kyoto University, Kyoto 606-8501, Japan \\ \emph{\texttt{hiraoka.yasuaki.6z@kyoto-u.ac.jp}}}

\author{Yusuke Imoto}
\address{\vspace{-0.5cm}Institute for the Advanced Study of Human Biology\\ Kyoto University Institute for Advanced Study\\ Kyoto University, Kyoto 606-8501, Japan \\ \emph{\texttt{imoto.yusuke.4e@kyoto-u.ac.jp}}}

\author{Théo Lacombe}
\address{\vspace{-0.5cm}Laboratoire d’Informatique Gaspard Monge,\\ Univ. Gustave Eiffel, CNRS, LIGM, F-77454 \\ Marne-la-Vallée, France \\ \emph{\texttt{theo.lacombe@univ-eiffel.fr}}}

\author{Killian Meehan}
\address{\vspace{-0.5cm}Institute for the Advanced Study of Human Biology\\ Kyoto University Institute for Advanced Study\\ Kyoto University, Kyoto 606-8501, Japan \\ \emph{\texttt{meehan.killianfrancis.8m@kyoto-u.ac.jp}}}

\author{Toshiaki Yachimura}
\address{\vspace{-0.5cm}Mathematical Science Center for Co-creative Society,\\ Tohoku University, \\Sendai 980-0845, Japan \\ \emph{\texttt{toshiaki.yachimura.a4@tohoku.ac.jp}}}


\begin{abstract}
\ntv{} is a graph embedding method that learns a vector representation for each node of a weighted graph while seeking to preserve relative proximity and global structure.
Numerical experiments suggest \ntv{} struggles to recreate the topology of the input graph.
To resolve this we introduce a \emph{topological loss term} to be added to the training loss of \ntv{} which tries to align the \emph{persistence diagram} (PD) of the resulting embedding as closely as possible to that of the input graph.
Following results in computational optimal transport, we carefully adapt \emph{entropic regularization} to PD metrics, allowing us to measure the discrepancy between PDs in a differentiable way.
Our modified loss function can then be minimized through gradient descent to reconstruct both the geometry and the topology of the input graph. 
We showcase the benefits of this approach using demonstrative synthetic examples.
\end{abstract}

\maketitle


\section{Introduction}
Various data types, such as bodies of text or weighted graphs, do not come equipped with a natural linear structure, complicating or outright preventing the use of most machine learning techniques that typically require Euclidean data as input. 
A natural workaround for this is to find a way to represent such data as sets of points in some Euclidean space $\mathbb{R}^m$. 
Regarding bodies of text, any unique word appears throughout a text with its own frequency and propensity for having certain neighbors or forming certain grammatical structures. This information can be used to assign each word some $m$-dimensional point in a way such that relative proximities of all these points maximally respect the neighborhood and structural data of the text. This is precisely the point of the \texttt{Word2vec} model \citep{word2vec}.

\ntv{} \citep{node2vec} and its predecessor \texttt{Deepwalk} \citep{deepwalk} learn representations of all the nodes in a weighted graph as Euclidean points in a fixed dimension. These are essentially \texttt{Word2vec} models in which each node of the input graph is given context in a `sentence' by taking random walks starting from that node. This random-walk-generated training data is then handled in precisely the same way \texttt{Word2vec} would handle words surrounded by sentences.
This construction is easy to alter and expand upon, allowing it to appear in complex projects like the one to embed new multi-contact (hypergraph) cell data, as seen in the general work \citep{hypersagnn} and then specialized to biological purposes in the follow-up paper \citep{matcha}.

While representation learning models prove useful for revitalizing existing analyses and opening up new insights, it is important to understand the extent of the data loss under such a transformation. In this paper, we identify an area of stark failure on the part of \ntv{} (or more generally, random-walk-based graph embeddings) to retain certain graph properties and so reintroduce the ability to resolve such features via the inclusion of a new loss function. This failure point and the loss function we introduce to compensate for it are both \emph{topological} in nature. 

\begin{figure}[h]
\captionsetup[subfigure]{justification=centering}
     \centering
     \begin{subfigure}[t]{0.17\textwidth}
         \centering
         \includegraphics[width=\textwidth]{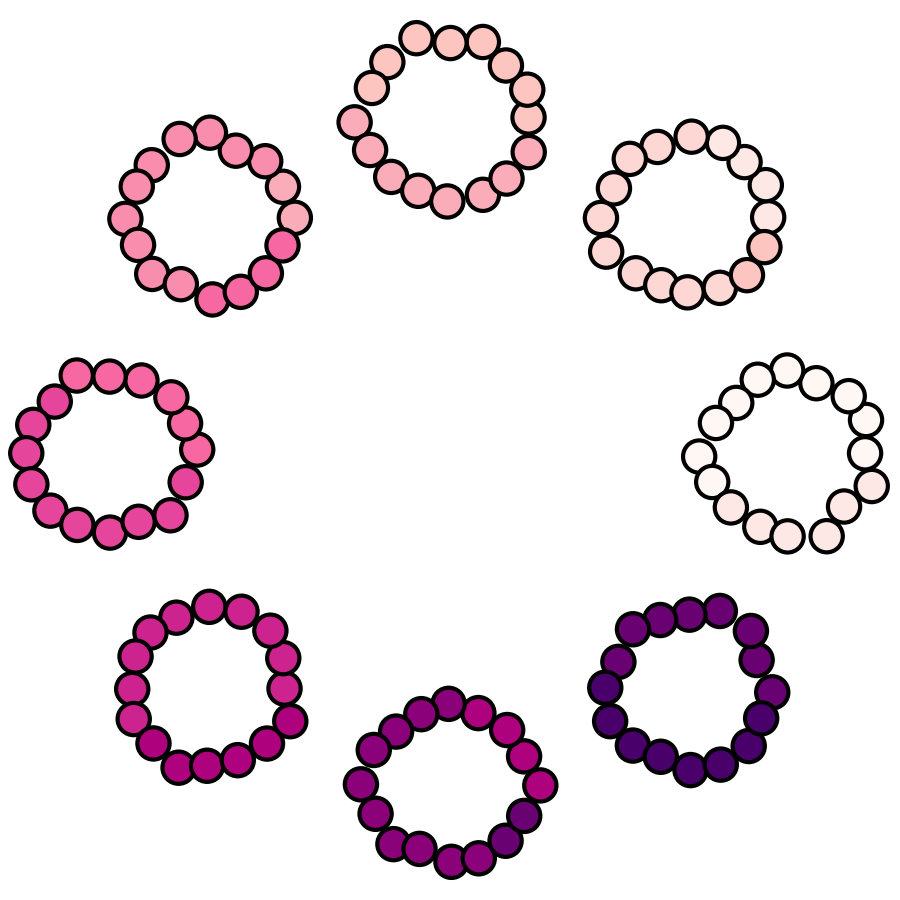}
         \caption{}
     \end{subfigure}
     \hfill
     \begin{subfigure}[t]{0.17\textwidth}
         \centering
         \includegraphics[width=\textwidth]{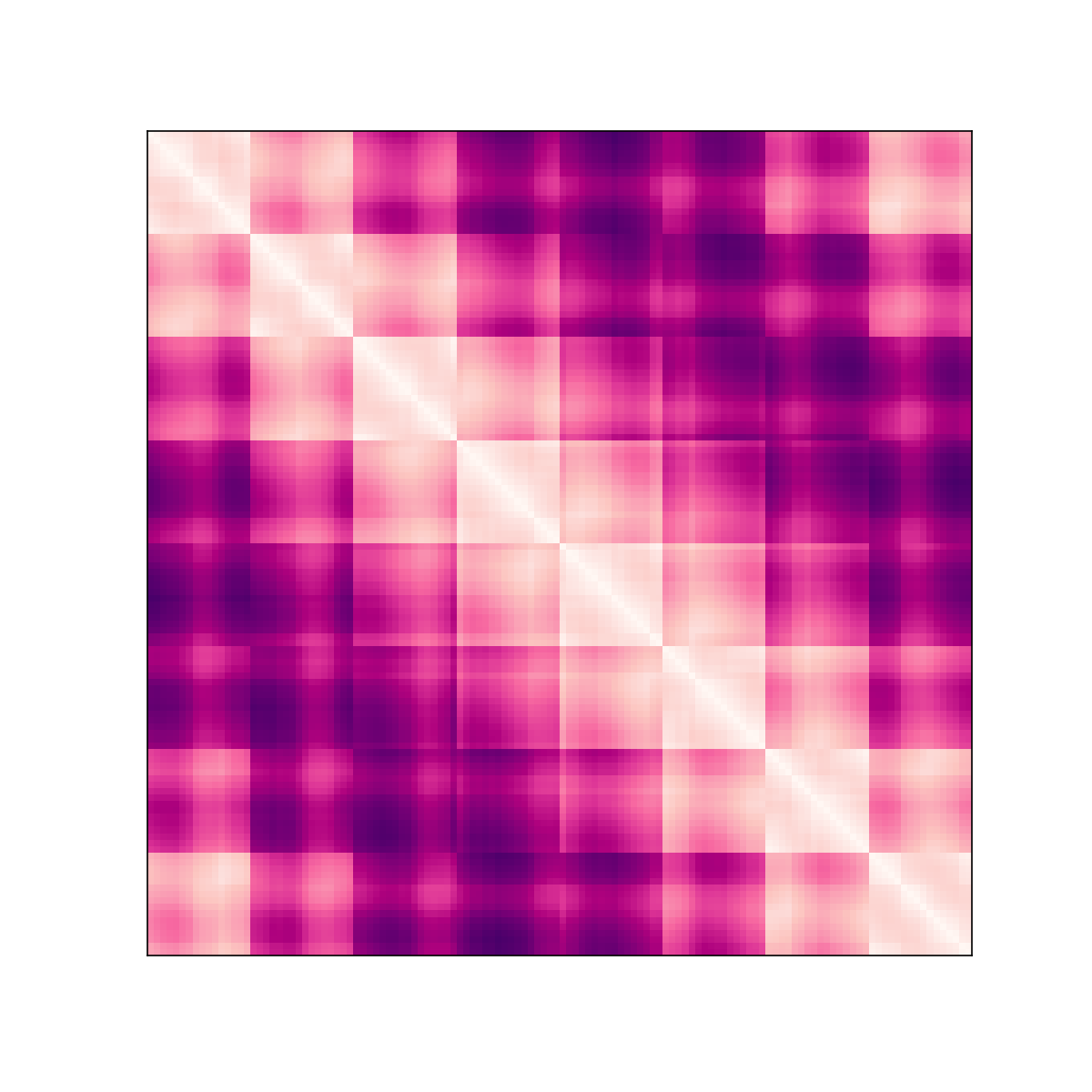}
         \caption{}
     \end{subfigure}
     \hfill
     \begin{subfigure}[t]{0.17\textwidth}
         \centering
         \includegraphics[width=\textwidth]{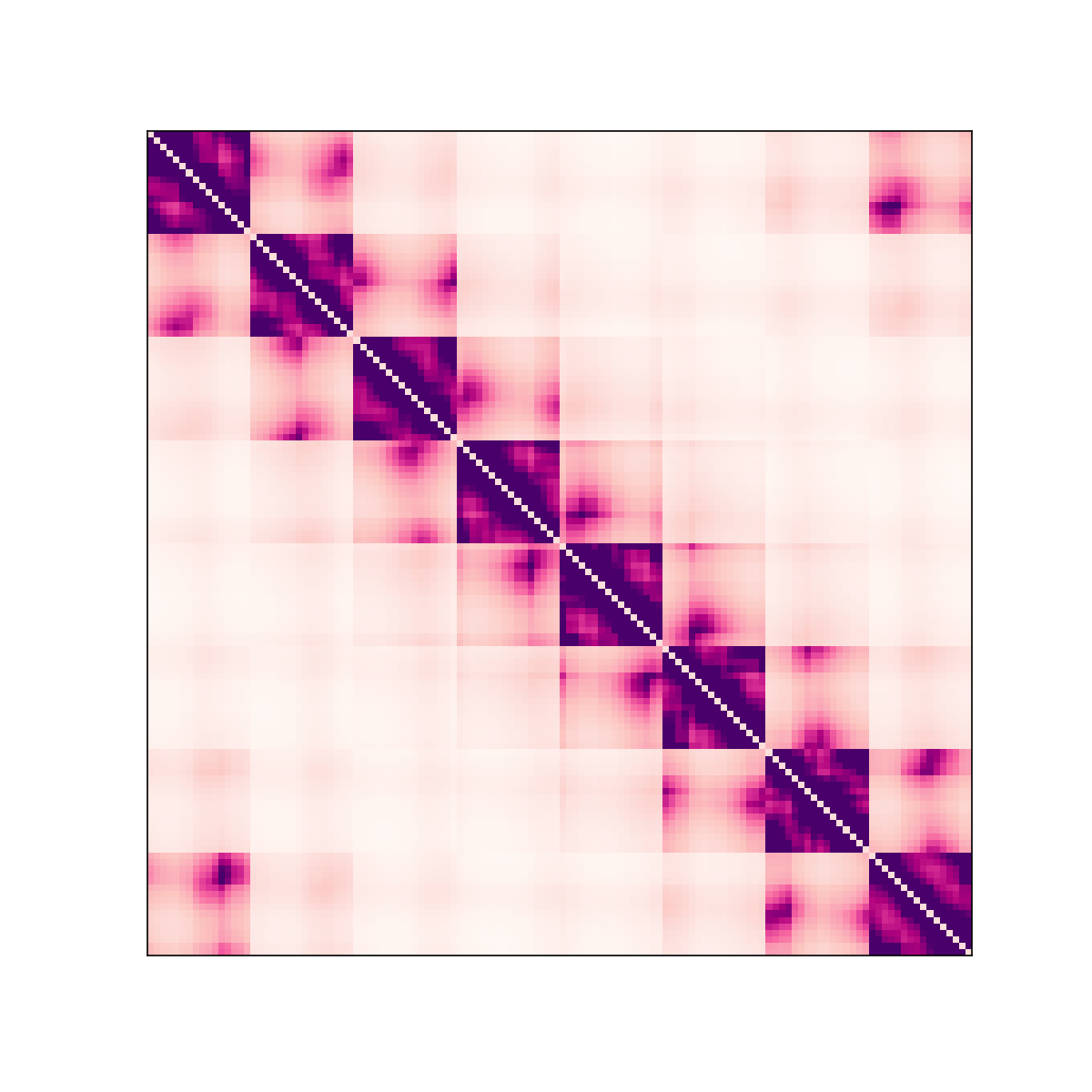}
         \caption{}
     \end{subfigure}
     \hfill
     \begin{subfigure}[t]{0.17\textwidth}
         \centering
         \includegraphics[width=\textwidth]{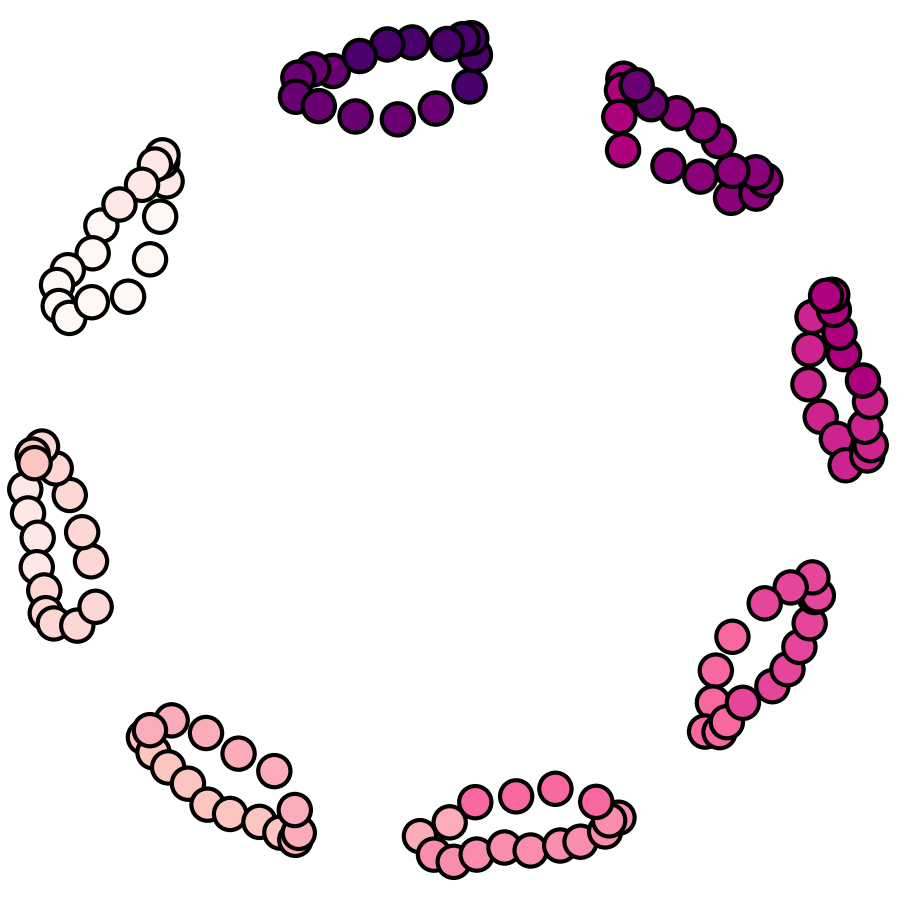}
         \caption{}
     \end{subfigure}
     \hfill
     \begin{subfigure}[t]{0.17\textwidth}
         \centering
         \includegraphics[width=\textwidth]{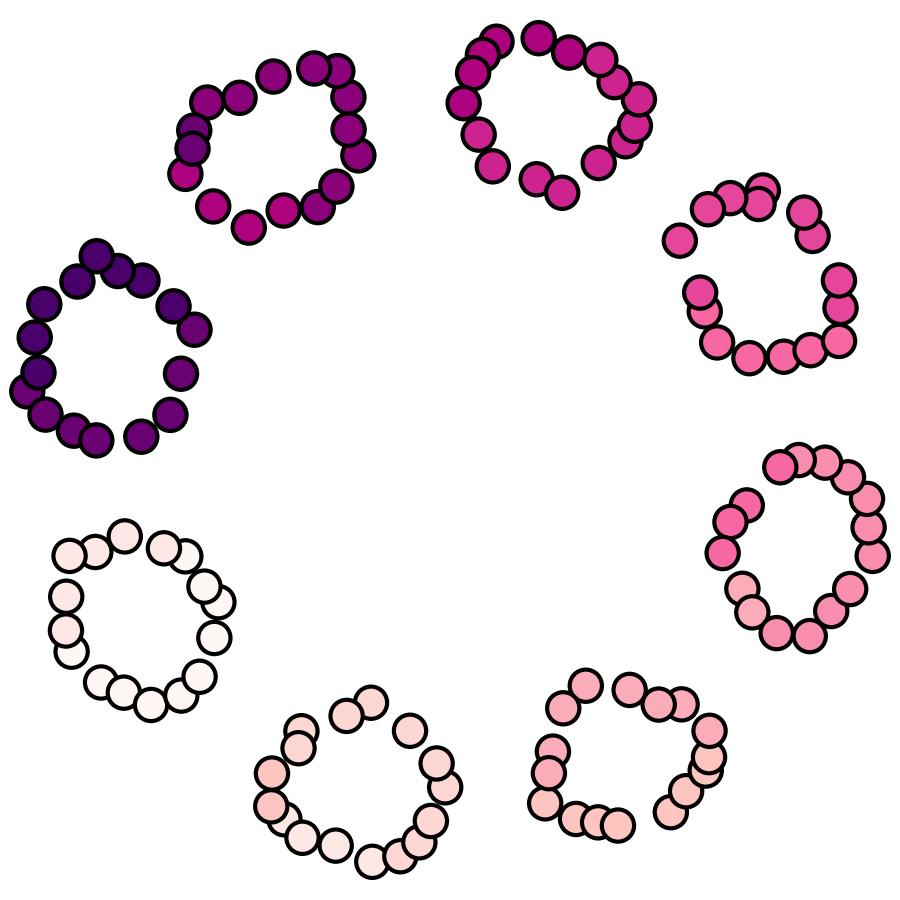}
         \caption{}
     \end{subfigure}
    \caption{Illustration of \ntv{} behavior with and without the incorporation of our topological loss during training. (a) An initial point cloud  and (b) its corresponding pairwise distance matrix. (c) The weighted adjacency matrix obtained by inverting the pairwise distances. This is the graph used as input for \ntv{} in this experiment. (d) The embedding proposed by \ntv{} after training using only the standard reconstruction loss. It fails to properly retrieve the eight smaller cycles appearing in the input graph, and the emergent central cycle is far too large. (e) The embedding proposed by \ntv{} after training while including our new topological loss term, in which the eight smaller cycles are recovered and the central cycle has been kept to a proper size.}
    \label{fig_start}
\end{figure}

\begin{figure}[h]
\captionsetup[subfigure]{justification=centering}
     \centering
     \begin{subfigure}[t]{0.33\textwidth}
         \centering
         \includegraphics[width=\textwidth]{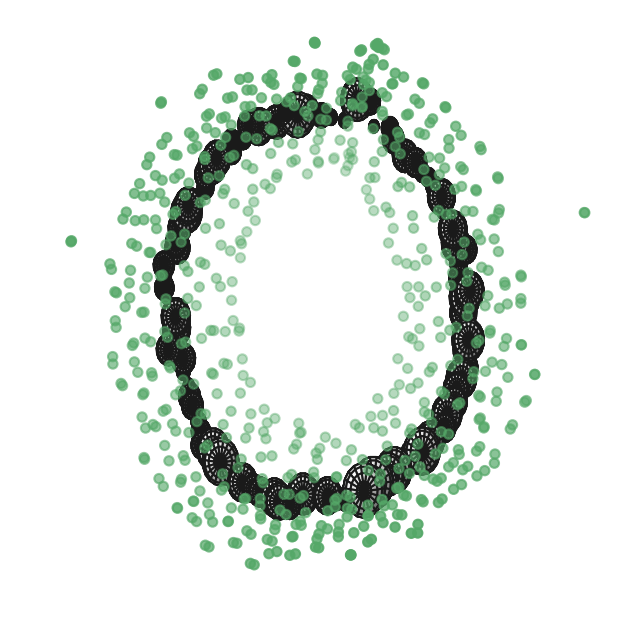}
         \caption{}
     \end{subfigure}
     \hspace{0.5cm}
     \begin{subfigure}[t]{0.33\textwidth}
         \centering
         \includegraphics[width=\textwidth]{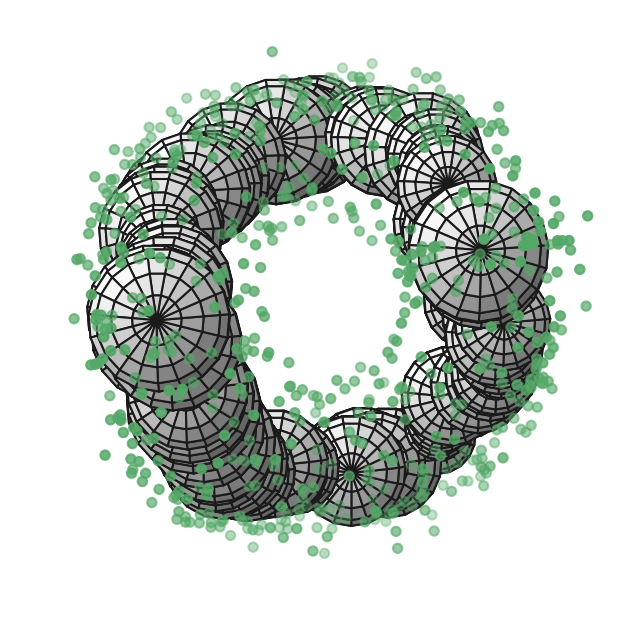}
         \caption{}
     \end{subfigure}
    \caption{Illustration of \ntv{} behavior on a \emph{torus} with and without the incorporation of our topological loss during training. (a) View of the final embedding proposed by \ntv{} using only the standard reconstruction loss, seen from `above'. The transparency of points represents the distance from the observer in the $z$-axis. Drawn in black are the largest inscribable spheres constructed at discretized angle increments around the center of the data set. The tubelike topology of the torus is all but erased. (b) View of the final embedding proposed by \ntv{} \emph{including our new topological loss term} with the largest inscribable spheres drawn inside. The tubelike topology of the torus is recovered.}
    \label{fig_start_torus}
\end{figure}

\subsection{Related Works}
\textbf{Graph embeddings.} We have already mentioned the existence and value of graph-representation models which use as training data some node neighborhood information. This training data can be prescribed via random walks \citep{deepwalk, node2vec} or other criteria \citep{LINE}, after which it is fed into a `skip gram' encoder framework as in \citep{word2vec} to learn a representation.

There are many node-to-vector representation learning models with various emphases or improvements on the general framework, for example: the work of \citep{deep-gaussian} which represents each node as a Gaussian distribution to capture uncertainty; MDS \citep{mds} which has existed for the better part of a century and presently refers to a category of matrix algorithms for representing an input distance matrix (reciprocal of a graph's adjacency matrix) as a set of Euclidean points; neighborhood methods that eschew a skip-gram model for direct matrix factorization \citep{grarep}; and \citep{atpge} leveraging data-motivated symmetries in directed graphs. We refer to \citep{cui2018survey, xu2021understanding} for comprehensive surveys of such methods.

While we focus on \ntv{} as an easily implemented, flexible, and widely used baseline, our method can be applied in parallel with \emph{any} graph representation method from which a \emph{persistence diagram} can be computed at every epoch. This is trivially the case for any method learning an Euclidean representation, such as node-to-vector representation models.

\textbf{Topological optimization.} In order to incorporate topological information into the training of \ntv{}, we rely on \emph{Topological Data Analysis} (TDA), a branch of algebraic topology rooted in the works \citep{persistence_and_simplification,carlsson_zomorodian}. 
\emph{Persistent homology} in particular provides topological descriptors called \emph{persistence diagrams} that summarize the topological information (connected components, loops, cavities, etc.) of structured objects (point clouds, graphs, etc.). These diagrams are frequently compared through partial matching metrics \citep{cohen2005stability,cohen2010lipschitz}. The question of \emph{optimizing} topology 
is introduced in \citep{continuation}. It has found different practical applications such as shape matching \citep{poulenard2018topological}, surface reconstruction \citep{bruel2020topologyAwareSurfaceReconstruction}, graph classification \citep{yim2021optimization}, and topological regularization of machine learning models \citep{chen2019topologicalRegularizer,hu2019topologyPreserving,gabrielsson2020layer}. 

While many applications see topology being applied to the abstract inner layers of machine learning frameworks for normalization or noise reduction \citep{survey}, in our application, as in the setting of \citep{auto}, every step of the learning process proposes a constantly improving Euclidean data-set from which we can directly obtain (and compare) topologies via persistent homology. 

Important theoretical results related to our work are \citep{optimizing}, which establishes the local convergence of stochastic gradient descents of loss functions that incorporate the comparison of topological descriptors, and \citep{leygonie2021framework} which provides a chain rule that enables explicit computation of gradients of topological terms. 

Topological optimization is an active research area and recent variants have been proposed to mitigate some of its limitations: \citep{leygonie2021gradient} shows that the non-smoothness of the persistent homology map can be alleviated by leveraging its stratified structure, and \citet{nigmetov2022topological} propose a way to address the sparsity of topological gradients. 
In this work, we propose another way of improving the behavior of topological loss functions by adding a specially designed entropic regularization term, an idea inspired by advancements in computational optimal transport.

\textbf{Optimal transport for TDA.} Optimal transport literature can be traced back to the works of \citet{monge1781memoire} and \citet{kantorovich2006translocation}. (We refer to the bibliography of \citep{villani2009optimal} for a comprehensive overview of optimal transport history.) The optimal transportation problem is a linear program whose goal is to optimize matchings, thus sharing important similarities with the metrics used in TDA to compare persistence diagrams \citep{mileyko2011probability, turner2014frechet}. This connection has been made explicit in \citep{Divol2021} and leveraged toward statistical applications in a series of works \citep{divol2021estimation,cao2022approximating}. During the last decade, the work of \citet{cuturi} has popularized the use of an \emph{entropic regularization} term (an idea that can be traced back to the work of \citet{schrodinger1932theorie}) which makes the resulting optimization problem strictly convex and hence more efficient to solve, along with ensuring differentiable solutions with respect to the input objects. This opened the door to a wide range of practical applications, see for instance \citep{solomon2015convolutional,benamou2015iterative}. 
The entropic optimal transport problem has been further refined through the introduction of so-called \emph{Sinkhorn divergences}, initially utilized on a heuristic basis 
\citep{ramdas2017wasserstein,genevay2018learning} and then further studied in \citep{feydy,sejourne2019sinkhorn}. A direct use of entropic optimal transport in TDA was first investigated in \citep{largescale}. 
However, it has recently been observed in \citep{lacombe2023hurot} that the problem introduced in the former work suffers from inhomogeneity. 
The latter, more recent work introduces a new regularized transportation problem that can be applied to TDA metrics while preserving the important properties of Sinkhorn divergences (efficient computation, differentiability) and which we further build upon in the present work. 

\subsection{Contributions and Outline}
Section \ref{sec_n2v} discusses preliminary technicalities required to properly introduce the \ntv{} model. 
We describe the simplest implementation of the \ntv{} model and carefully derive the gradient of the training loss.
Section \ref{sec_topological_loss_n2v} introduces the topological loss function which will be used to incorporate topological information into the \ntv{} model. 
In particular, we apply a recently introduced \emph{entropic regularization} of metrics used to compare topological descriptors based on ideas developed in the computational optimal transport community and carefully derive the corresponding gradient. 
To the best of our knowledge, this is the first use of entropic regularization in the context of topological optimization. 
Gathering these results, \cref{sec_tn2v} presents the modified algorithm we use to train this new \emph{Topological} \ntv{} model. 
We provide numerical results in Section \ref{sec_experiments}, elaborating on the results in Figures \ref{fig_start} and \ref{fig_start_torus} to demonstrate that the addition of topological information significantly improves the quality of embeddings. 
Our implementation is publicly available at this \href{https://github.com/killianfmeehan/topological_node2vec}{repository}\footnote{\url{https://github.com/killianfmeehan/topological_node2vec}}.

\section{The Node2vec Model}\label{sec_n2v}

\ntv{} is a machine learning model that learns the representation of a graph's nodes as a set of Euclidean points in some specified dimension. We carefully explain the neighborhood generation methodology characteristic to \ntv{} (based on \texttt{Word2vec} and \texttt{Deepwalk}) and the structure of \ntv{}'s internal parameters, as well as derive the gradient of its loss function.

We elaborate on the details of \ntv{} in this section following the structure laid out in Algorithm \ref{alg_node2vec}. 
\ntv{} learns an embedding in $\mathbb{R}^m$ of a graph by minimizing a loss function $L_0$ that measures some distance between the currently proposed embedding (initially random) and training data extracted from the input graph.

\algnewcommand\algorithmicinput{\textbf{Input:}}
\algnewcommand\INPUT{\item[\algorithmicinput]}

\algnewcommand\algorithmicoutput{\textbf{Output:}}
\algnewcommand\OUTPUT{\item[\algorithmicoutput]}

\algnewcommand\algorithmicstart{\textbf{Start:}}
\algnewcommand\start{\item[\algorithmicstart]}

\begin{algorithm}
\caption{Node2vec algorithm}\label{alg_node2vec}
\begin{algorithmic}
    \INPUT a weighted graph $G=(V,E,w)$, learning rate $\eta > 0$, embedding dimension $m$, neighborhood parameters $(l,r,p,q)$;
    \OUTPUT $\mathbb{P}\subset\mathbb{R}^m$ with $|\mathbb{P}|=|V|$
    \start
    \State • initialize the parameter matrices $\Theta(0)$ (Definition \ref{def_structure}) with values uniformly sampled from the interval $(-1,1)$; let $k=0$;
    \While{not converged}
    \State • for each node $v\in V$, generate the \emph{current neighborhood vector} $C_v$ (Definition \ref{def_predicted}) and the \emph{training neighborhood vector} $T_v$ (Eq.~\eqref{eq:Tv}); let $C(k) = \{C_v(k)\}_{v \in V}$ and $T(k):=\{T_v(k)\}_{v\in V}$
    \State • update $\Theta(k+1)$ using the gradient result of Proposition \ref{prop_ce_gradient}: $$\Theta(k+1) := \Theta(k) - \eta\nabla_{\Theta(k)} L_0(T(k),C(k))$$
    \State • $k\mathrel{+}=1$
    \EndWhile
\end{algorithmic}
\end{algorithm}

We consider in this work weighted, undirected graphs, referred to simply as \emph{graphs}, as represented by a tuple $G=(V,E,w)$, where $V=\{v_1,\ldots,v_n\}$ is a set of vertices with a total order, $E=V\times V$, and $w:E\to\mathbb{R}_{\geq 0}$ is a \emph{symmetric} weight function (i.e., $w(v_i,v_j)=w(v_j,v_i)$ for all $1\leq i<j\leq n$) into the set of non-negative real numbers.

\begin{definition}[\ntv{} structure] \label{def_structure}
Let $G=(V,E,w)$ be a graph, $|V|=n$ denote the number of nodes, and $m\in\mathbb{N}$ be the desired embedding dimension. 
Define $\Theta :=\{W_1 ,W_2 \}$ where  $W_1 \in\mathcal{M}_{n\times m}(\mathbb{R})$ and $W_2 \in\mathcal{M}_{m\times n}(\mathbb{R})$ are matrices of size $n \times m$ and $m \times n$, respectively.
Let $W_1(i,\cdot)$ denote the $i$\textsuperscript{th} row of the matrix $W_1$ for $1 \leq i \leq n$. 
The corresponding \ntv{} embedding is defined as 
\[\mathrm{Emb}(\Theta) = \{ W_1(i,\cdot) \}_{1 \leq i \leq n} \subset \mathbb{R}^m. \] 
When necessary, we will write $\Theta(k), \mathrm{Emb}(\Theta(k))$ to convey dependence on the epoch $k$ during the learning process.
\end{definition}

\begin{definition}[Current neighborhood probability] \label{def_predicted}
For $v=v_i\in V$, given some parameter matrices $\Theta$, the \emph{current predicted neighborhood probability of }$v$ is the vector $C_{v}$ given by the following: Let $\bar{v} = (0,\ldots,0,1,0,\ldots,0)\in\mathbb{R}^n$ be the row vector where $\bar{v}(i) = 1$, and for all $1\leq j\leq n$ with $j\neq i$, $\bar{v}(j) = 0$. 
Define 
\begin{equation}\label{eq:presoftmax}
    u (= u_v) := \bar{v}\cdot {W_1}\cdot {W_2}\in\mathbb{R}^n
\end{equation} and subsequently let $C_v$ be the softmax of $u$: 
\[C_{v} = \left(\frac{e^{u(1)}}{\sum_{1\leq j\leq n} e^{u(j)}},\ldots, \frac{e^{u(n)}}{\sum_{1\leq j\leq n} e^{u(j)}}\right).\]
\end{definition}

For a node $v\in V$, the \emph{training neighborhood probability} vector $T_v$ is generated through the following process: 
Four parameters are needed to dictate how $T_v$ is generated:
\begin{itemize}
    \item $r$: the number of random walks to be generated starting from the node $v$,
    \item $l$: the length of each walk,
    \item $p$: a parameter that determines how often a walk will return to the previous vertex in the walk,
    \item $q$: a parameter that determines how often a walk will move to a new vertex that is not a neighbor of the previous vertex.
\end{itemize}

We generate $r$ walks starting from $v$, each of length $l$, with probability transitions given the parameters $p,q$ and weight function $w:E\to\mathbb{R}_{\geq0}$.
Specifically:
\begin{itemize}
    \item For the first step of the walk, the probability of traveling to any node $v_\textrm{next}$ is the normalized edge weight $$\frac{w(v,v_\textrm{next})}{\sum_{\hat{v}\in V}w(v,\hat{v})}.$$
    \item Suppose the previous step in the walk was $v_\textrm{prev}\to v_\textrm{curr}$. For any node $v_\textrm{next}\in V$, define
    \begin{align*}\xi(v_\textrm{prev},v_\textrm{curr},v_\textrm{next}) =
    \left\{\begin{array}{ll}
    0 & \text{if }w(v_\textrm{curr},v_\textrm{next})=0 \\
    1/p & \text{if }w(v_\textrm{curr},v_\textrm{next})\neq0\text{ and }v_\textrm{prev}=v_\textrm{next} \\
    1 & \text{if }w(v_\textrm{curr},v_\textrm{next})\neq0\text{, }v_\textrm{prev}\neq v_\textrm{next}\text{, and }w(v_\textrm{prev},v_\textrm{next})>0 \\
    1/q & \text{otherwise}.
    \end{array}\right.
\end{align*}
\end{itemize}

Then, the probability of choosing any $v_\textrm{next}$ as the next node in the walk is 
\begin{equation}\label{eq:transition}
\frac{{\xi(v_\textrm{prev},v_\textrm{curr},v_\textrm{next})}\cdot w(v_\textrm{curr},v_\textrm{next})}{\sum_{\hat{v}\in V}\xi(v_\textrm{prev},v_\textrm{curr},\hat{v})\cdot w(v_\textrm{curr},\hat{v})}.
\end{equation}

Collect all the nodes traversed by the $r$ random walks starting from $v$ and save this information in a multiplicity function $t_v:V\to\mathbb{Z}_{\geq 0}$. That is, for any node $v_j\in V$, $t_v(v_j)$ is equal to the number of times this node was reached across all the random walks starting from $v$. 

Finally, represent the information of the multiplicity function $t$ as a probability vector indexed over all the nodes of $V$: 
\begin{equation}\label{eq:Tv}
    T_v=\left(\frac{t_v(v_j)}{l\cdot r}\right)_{1\leq j\leq n}.
\end{equation} When necessary, we will write this vector as $T_v(k)$ to convey dependence on the epoch $k$ during the learning process.

\begin{remark}\label{rmk_r_infinity}
    It is also possible to neglect this random walk process and simply take $T_v$ to be the normalization of the vector $(w(v_1,v),\ldots,w(v_n,v)).$ This is equivalent to the parameter choice of $l=1$ and $r=\infty$. The choice of $p$ and $q$ have no impact on walks of length $l=1$. 
    This captures \emph{immediate} neighborhood information and fails to see structures such as hubs and branches. However, as showcased later in our experiments, there are situations in which this is a computationally superior option.
\end{remark}

\ntv{} is trained by minimizing the following loss $L_0$: 
\begin{equation}\label{eq:n2v_std_loss}
    \Theta \mapsto L_0(T,\Theta) :=\sum_{v\in V}H(T_v,C_v),
\end{equation}
where $H$ denotes the cross-entropy between probability distributions, given by 
$$H(A,B)= -\sum_{y\in Y} A(y)\log B(y)$$
with the convention that $0 \log(0)=0$. 
A practical formula to express this loss is given by the following lemma: 

\begin{lemma}[\citet{node2vec}, Equation (2)]
Let $v\in V$. Treating $C_v$ and $T_v$ as probability distributions over $V$,
\begin{equation*}
    H(T_v,C_v) = -\sum_{1\leq j\leq n}T_v(j)u_v(j) + \log\sum_{1\leq j\leq n} e^{u_v(j)},\label{losssimplified}
\end{equation*}
where $u_v$ is as defined in \eqref{eq:presoftmax}.
\end{lemma}

The minimization of \eqref{eq:n2v_std_loss} is performed through gradient descent, for which we provide an explicit expression below.
In the following proposition, we use $M(\cdot,\cdot)$ to denote matrix coordinates in order to avoid excessive subscripting. We also make use of the Kronecker delta notation where for comparable objects $a,b$, $\delta_{a,b}=1$ if $a=b$ and $0$ otherwise.

\begin{proposition}\label{prop_ce_gradient}
Recall that $\Theta = (W_1,W_2)$. Then, the gradient of $L_0$ with respect to $W_1$ is $$\nabla_{W_1}L_0(T,\Theta) = 
\left[\sum_{1\leq i\leq n}
\delta_{v_i,v_l}\cdot\sum_{1\leq a\leq n}W_2(j,a)(C_{v_i}(a)-T_{v_i}(a))\right]_{\substack{1\leq l\leq n\\1\leq j\leq m}}.$$
Also, the gradient of $L_0$ with respect to $W_2$ is
$$\nabla_{W_2}L_0(T,\Theta) = 
\left[\sum_{1\leq i\leq n}W_1(i,j)(C_{v_i}(l)-T_{v_i}(l))\right]_{\substack{1\leq l\leq n\\1\leq j\leq m}}.$$
\end{proposition}

The proof of this result is deferred to the appendix. 

\section{A Topological Loss Function for \ntv{}}\label{sec_topological_loss_n2v}
This section presents the important background on topological data analysis required in this work, including the computation of topological descriptors called \emph{persistence diagrams} (PD) and their comparison through \emph{entropy regularized metrics}.
Eventually, we derive an explicit gradient for these metrics that we will use when training \ntv{} with a topological term in its reconstruction loss.

Essentially, we will include a \emph{topological loss} term $L_1$ when training \ntv{} (see Eq.~\eqref{eq:n2v_std_loss}) which penalizes the difference between the PD of \ntv{}'s output and the PD of the input graph.

\subsection{From Point Clouds to Persistence Diagrams}\label{sec_ph}
Persistent homology identifies topological features in structured objects such as point clouds and encodes them as discrete measures called \emph{persistence diagrams}, supported on the open half-plane $\Omega = \{(b,d) \in \mathbb{R}^2,\ d > b\}$. 
Each point in these measures expresses the scales within the original point cloud at which a given topological feature was created (born) and destroyed (died).
We refer to \citep{chazal2016structure, edelsbrunner2022computational} for a general introduction. 
Below, we provide a concise presentation of this theory specialized to the context of our work. 

Let $\mathbb{P}=\{p_1, \ldots, p_n\} \subset \mathbb{R}^m$ be a finite point cloud.
The Vietoris-Rips complex of parameter $\vrparameter$, denoted $\mathrm{VR}_\vrparameter$, is the simplicial complex over $\mathbb{P}$ given by all simplices such that \emph{the longest pairwise distance between the vertices of that simplex} is at most $\vrparameter$.
Increasing $\vrparameter$ from $0$ to the diameter of the point cloud gives us the filtration of simplicial complexes $\mathrm{VR}=\{\mathrm{VR}_\vrparameter\}$. 
A topological feature is one that is enclosed at some $\vrparameter_\textrm{birth}$ (a hollow space enclosed by edges, a hollow void enclosed by triangles, etc.) and then eventually filled in by higher dimensional simplices at some $\vrparameter_\textrm{death}>\vrparameter_\textrm{birth}$.
We call $\vrparameter_\textrm{birth}$ and $\vrparameter_\textrm{death}$ the \emph{birth} and \emph{death} values of the topological feature, respectively.

\begin{figure}
    \centering
    \includegraphics[width=1.00\textwidth]{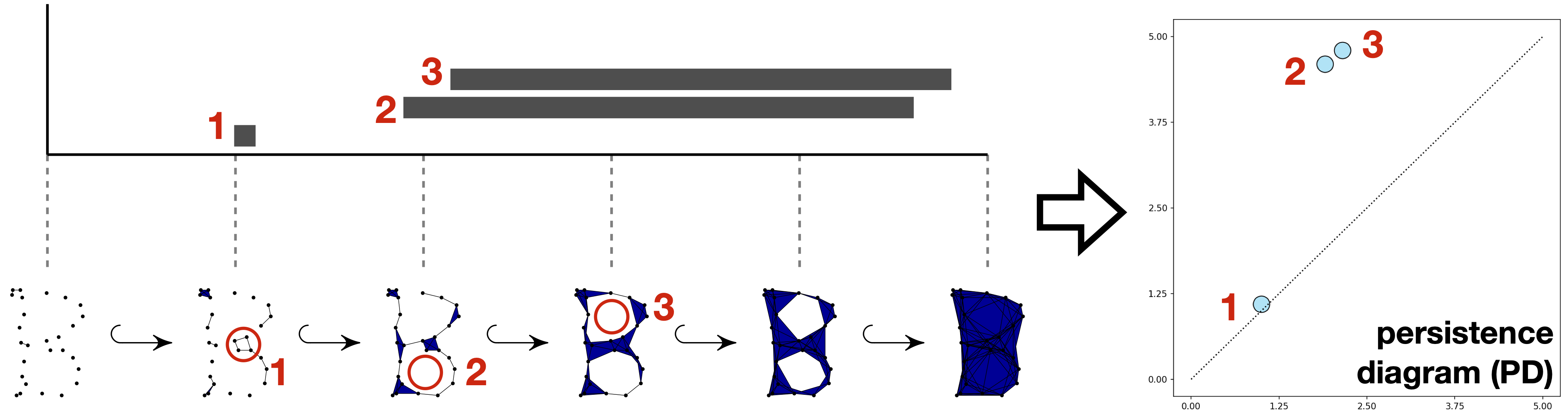}
    \caption{Vietoris-Rips construction of the filtered simplicial complex on a 2D point cloud. The \emph{persistence diagram} represents topological features as points in the upper diagonal half-plane of $\mathbb{R}^2$. Points close to the diagonal such as point $1$ above (birth and death values are very close) are regarded as noise. Features far from the diagonal, such as points $2$ and $3$, \emph{persist} across a large range of scales and are considered significant.}
    \label{fig_persistent_homology}
\end{figure}

We say a point cloud is in Vietoris-Rips general position when all of the birth and death values of all topological features are \emph{unique}. 
The \emph{persistence diagram} of $\mathbb{P}$ is the multiset $\mathrm{PD}(\mathbb{P})=\{x_i=(b_{x_i},d_{x_i})\}_{1\leq i\leq N}\subset\Omega$ collecting all the birth and death values of topological features of $\mathbb{P}$ (See Figure \ref{fig_persistent_homology}). Equivalently, this information can be represented as the finite \emph{counting measure} $\alpha = \sum_{i=1}^N \delta_{x_i}$, where 
$\delta_{x_{i}}$ denotes the Dirac mass located at $x_{i}$.

\subsection{Metrics for PDs and Optimal Transport}\label{sec_ot}
We now elucidate the development of metrics in optimal transport (OT) regarding how they can be applied to measuring the distance between persistence diagrams (PDs). In the following, we adopt the representation of PDs as finite counting measures, a formalism initially introduced in \citep{chazal2016structure} which is vital for discussing connections to OT literature. 

We denote the orthogonal projection of $x \in \Omega$ to the diagonal $\thediag := \{(b,b),\ b \in \mathbb{R} \}$ by $\projdiag(x)$.

\begin{definition}
    The space $\mathcal{D}$  of persistence diagrams is the set of finite counting measures supported on $\Omega$, that is 
    \begin{equation*}
        \mathcal{D} = \left\{ \alpha = \sum_{i=1}^N \delta_{x_i} \ \big| \ x_i\in\Omega, N\in\mathbb{N} \right\}.
    \end{equation*}
    We equip $\mathcal{D}$ with the metric $(\alpha, \beta) \mapsto \FG(\alpha,\beta)^{\frac{1}{2}}$ \citep{cohen2010lipschitz} expressed as the square-root of
    \begin{equation}\label{eq:FG_metric}
        \FG(\alpha,\beta) = \inf_{\zeta \in \Xi(\alpha,\beta)} \sum_{x} \|x - \zeta(x)\|^2,
    \end{equation}
    where $\Xi(\alpha,\beta)$ denotes the set of bijections from $\spt(\alpha) \cup \partial \Omega$ to $\spt(\beta) \cup \partial \Omega$, and $\spt(\mu)$ denotes the support of a counting measure $\mu$ (i.e.~the set of all $x$ such that $\mu(\{x\}) >0$). 
    Such a $\zeta$ essentially bijects a subset of $\alpha$ onto a subset of $\beta$ while mapping all leftover points to the diagonal $\partial \Omega$. We call $\zeta$ a \emph{partial matching} between $\alpha$ and $\beta$.
    Any $\zeta^* \in \Xi(\alpha,\beta)$ that achieves the infimum in \eqref{eq:FG_metric} is said to be an \emph{optimal partial matching} between $\alpha$ and $\beta$.
\end{definition}

\begin{remark}
    In practice, persistence diagrams come with an \emph{essential component} which consists of points whose death coordinate is $+\infty$. 
    In the context of Vietoris-Rips filtrations built on top of arbitrary point clouds, there is exactly one such point in the persistence diagram of the form $(0,+\infty)$.
    We ignore this for the purposes of our analysis.
\end{remark}

\begin{remark}
In the TDA literature, the metric $\FG(\cdot,\cdot)^{\frac{1}{2}}$ is often referred to as the \emph{Wasserstein distance between persistence diagrams} (see, e.g., \citep{cohen2010lipschitz,turner2014frechet}), using an analogy with the Wasserstein distance introduced in OT literature \citep{villani2009optimal,santambrogio2015optimal,peyre2019computational}. 
A formal connection between these two concepts was established in \citep{Divol2021}, where it was shown that the metric used in TDA is equivalent to the \emph{optimal transport problem with Dirichlet boundary condition} introduced by \citet{Figalli2010}. 
We use the notation $\FG$ to stress that 
this metric is 
\emph{not} equivalent to the Wasserstein distance used in OT. 
Adapting tools from OT literature to PDs (entropic regularization, gradients) will require this level of technical care.
\end{remark}

The following reformulation of $\FG$ will be convenient going forward: 
\begin{proposition}[\citep{Divol2021,lacombe2023hurot}] \label{prop:from_matching_to_matrix}
Let $\alpha,\beta \in \mathcal{D}$ be two persistence diagrams with $\alpha = \sum_{i=1}^N \delta_{x_i}$ and $\beta = \sum_{j=1}^M \delta_{y_j}$. The (squared) distance between $\alpha$ and $\beta$ can be rephrased as the following minimization problem:
\begin{equation}\label{eq:OT_TDA}
\begin{aligned}
    \FG(\alpha,\beta) = \inf_{P \in \Pi} & \bigg[ \sum_{\substack{1 \leq i \leq N,\\ 1 \leq j \leq M}} \|x_i - y_j\|^2 P_{ij} \\
    &+ \sum_{i=1}^N \|x_i - \projdiag(x_i)\|^2 \left(1 - \sum_{j=1}^M P_{ij}\right) \\
    &+ \sum_{j=1}^M \|y_j - \projdiag(y_j)\|^2 \left(1 - \sum_{i=1}^N P_{ij} \right) \bigg],
\end{aligned}
\end{equation}
where
\begin{equation}\label{eq:submarginal_polytope}
    \Pi = \left\{ P \in \mathbb{R}_{\geq 0}^{N \times M},\ \forall i \in \{1,\dots,N\}, \forall j \in \{1,\dots,M\},\ \sum_{i'=1}^N P_{i'j} \leq 1, \sum_{j'=1}^M P_{ij'} \leq 1 \right\}.
\end{equation}
\end{proposition}

The main claim of \cref{prop:from_matching_to_matrix} is that the infimum over bijections $\zeta$ in \eqref{eq:FG_metric} can be relaxed to a minimization over matrices $P$ that satisfy \emph{sub-marginal} constraints as a consequence of Birkoff's theorem. 
Such a $P$ will be referred to as a \emph{partial transport plan}.

\subsection{Entropic Regularization for PD Metrics}
While theoretically appealing, the metric $\FG$ has drawbacks. 
It is computationally expensive and, though the optimal matching $\zeta^*$ is generically unique, it suffers from instability: a small perturbation in $\alpha$ can significantly change $\zeta^*$, which is a clearly undesirable behavior when it comes to optimizing topological losses (as detailed in section \ref{sec_vr}).

In computational optimal transport, an alternative approach popularized by \citet{cuturi} consists of introducing a \emph{regularization term} based on \emph{entropy}. An illustrative translation of this idea into the context of PD metrics would be to consider the following adjusted problem from Proposition \ref{prop:from_matching_to_matrix}:
\begin{equation}\label{eq:naive_OT_reg}
\begin{aligned}
    \text{minimize } P \mapsto \bigg[ &\sum_{\substack{1 \leq i \leq N,\\ 1 \leq j \leq M}} \|x_i - y_j\|^2 P_{ij} \\ 
    &+ \sum_{i=1}^N \|x_i - \projdiag(x_i)\|^2 \left(1 - \sum_{j=1}^M P_{ij}\right) \\
    &+ \sum_{j=1}^M \|y_j - \projdiag(y_j)\|^2 \left(1 - \sum_{i=1}^N P_{ij} \right) \\
    &+ \entropicregparam \KL(P | 1_{NM}) \bigg],
\end{aligned}
\end{equation}
where $P \in \Pi$, $\entropicregparam > 0$ is the `regularization parameter', $1_{NM}$ denotes the matrix of size $N \times M$ filled with $1$s, and $\displaystyle \KL(A | B) = \sum_{\substack{1 \leq i \leq N, \\ 1 \leq j \leq M}} A_{ij} \log\left( \frac{A_{ij}}{B_{ij}} \right) - A_{ij} + B_{ij}$ for any two matrices $A,B \in \mathbb{R}^{N \times M}$. 
In \citep{largescale} it was shown that, just like its OT counterpart, \eqref{eq:naive_OT_reg} can be solved by the Sinkhorn algorithm \citep{sinkhorn1964relationship}, a simple iterative algorithm that is highly parallelizable and GPU-friendly.

However, an important observation made in \citep{lacombe2023hurot} is that \eqref{eq:naive_OT_reg} is not $1$-homogeneous in $(\alpha,\beta)$ when $\entropicregparam > 0$ (while the same equation with $\entropicregparam = 0$ \emph{is}). While this is of no consequence to the typical OT practitioner (see \cite[Section~3]{lacombe2023hurot}), it dramatically affects persistence diagrams, which may have highly disparate total masses. Loosely speaking, when $N, M$ are large, the entropic contribution tends to outweigh the transport cost in \eqref{eq:naive_OT_reg}, yielding transportation plans that concentrate near the diagonal. To overcome this behavior, we propose the following version of entropic regularization for PD metrics, building on Homogeneous Unbalanced Regularized OT (HUROT) \citep{lacombe2023hurot}:

\begin{definition}\label{def_hurot}
Let $\alpha=\sum_{i=1}^N\delta_{x_i},\beta =\sum_{j=1}^M\delta_{y_j}\in \mathcal{D}$  and $\entropicregparam > 0$. 
The HUROT problem $\FG_\entropicregparam$ between $\alpha$ and $\beta$ is defined as
\begin{equation}\label{eq:HUROT_TDA}
\begin{aligned}
    \FG_\entropicregparam(\alpha,\beta) := \inf_{P \in \Pi} & \bigg[ \sum_{\substack{1 \leq i \leq N, \\ 1 \leq j \leq M}} \|x_i - y_j\|^2 P_{ij} \\ 
    &+ \sum_{i=1}^N \|x_i - \projdiag(x_i)\|^2 \left(1 - \sum_{j=1}^M P_{ij}\right) \\
    &+ \sum_{j=1}^M \|y_j - \projdiag(y_j)\|^2 \left(1 - \sum_{i=1}^N P_{ij} \right) \\
    &+ \entropicregparam R(P|\mathbf{a} \mathbf{b}^T) \bigg],
\end{aligned}
\end{equation}
 where $\Pi$ is defined in \eqref{eq:submarginal_polytope} and the vectors $\mathbf{a} = (\|x_i - \projdiag(x_i)\|^2)_{i=1}^N \in \mathbb{R}^N$ and $\mathbf{b} = (\|y_j - \projdiag(y_j)\|^2)_{j=1}^M \in \mathbb{R}^M$ are both treated as column vectors, resulting in $\mathbf{a} \mathbf{b}^T$ being a $N \times M$ matrix. The homogeneous entropic regularization term is given by
\begin{equation}\label{eq:hurot_entropic_term}
    R(P | \mathbf{a} \mathbf{b}^T) := \frac{1}{2}\left( \KL\left(P \bigg| \frac{\mathbf{a} \mathbf{b}^T}{\pers(\alpha)}\right) + \KL\left(P \bigg| \frac{\mathbf{a} \mathbf{b}^T}{\pers(\beta)}\right)  \right)
\end{equation}
where $\displaystyle \pers(\alpha) = \sum_{i=1}^N \|x_i - \projdiag(x_i)\|^2$ is called the \emph{total persistence} of $\alpha$ (and similarly for $\beta$). 

We additionally define the \emph{Sinkhorn divergence between persistence diagrams} to be
\begin{equation}\label{eq:HUROT_SFG}
    \SFG_\entropicregparam(\alpha,\beta) := \FG_\entropicregparam(\alpha,\beta) - \frac{1}{2} \FG_\entropicregparam(\alpha,\alpha) - \frac{1}{2} \FG_\entropicregparam(\beta,\beta).
\end{equation}
\end{definition}

The main benefit of working with $\FG_\entropicregparam$ instead of $\FG$
is that the former is defined through a strictly convex optimization problem (while being $1$-homogeneous in $(\alpha,\beta)$, a major improvement over \eqref{eq:naive_OT_reg} used in \citep{largescale}). 
In addition to what we gain in terms of computational efficiency, we ensure that the optimal partial transport plan $P^\entropicregparam_{\alpha,\beta}$ between $\alpha$ and $\beta$ for this regularized problem is unique and smooth in $(\alpha, \beta)$.

Although $\FG_\entropicregparam$ approximates $\FG$ in a controlled way (in particular when $\entropicregparam \to 0$, see \citep{altschuler2017near,dvurechensky2018computational,pham2020unbalanced}), it also introduces an unavoidable \emph{bias}: namely, while $\FG(\cdot,\cdot)^{\frac{1}{2}}$ is indeed a metric in $\mathcal{D}$, in general we have that $\FG_\entropicregparam(\beta,\beta) > 0$. 
It is even the case that $\min_\alpha \FG_\entropicregparam(\alpha,\beta) < \FG_\entropicregparam(\beta,\beta)$ \citep{feydy}, meaning that if we minimize the map $\alpha \mapsto \FG_\entropicregparam(\alpha,\beta)$ through a gradient-descent-like algorithm, $\alpha$ is not pushed ``toward $\beta$'', but rather an inwardly shrunken version of it (see for instance \citep{janati2020debiased}). 

Finally, this bias is compensated for in \eqref{eq:HUROT_SFG} by our introduction of \emph{Sinkhorn divergence for persistence diagrams}, which follows parallel ideas developed in OT  \citep{ramdas2017wasserstein,genevay2018learning,feydy}. 
In the case of persistence diagrams, it was proved in \citep{lacombe2023hurot} that under relatively mild assumptions we can guarantee $\SFG_\entropicregparam(\alpha,\beta) \geq 0$, with equality if and only if $\alpha = \beta$. 
Furthermore, $\SFG_\entropicregparam(\alpha_n,\beta) \to 0 \Leftrightarrow \FG(\alpha_n,\beta) \to 0$ for any $\beta \in \mathcal{D}$ and any sequence of persistence diagrams $(\alpha_n)_n$. 

\subsection{Gradient of the Vietoris-Rips Map}\label{sec_vr}
Recall from Section \ref{sec_ph} that in the Vietoris-Rips construction of persistence diagrams any point $x=(b_x,d_x)\in \Omega$ in the PD of $\mathbb{P}=\{p_1,\ldots,p_n\}$ represents some topological feature, where $b_x$ is its birth value and $d_x$ is its death value.
These values correspond to unique (assuming Vietoris-Rips general position) simplices that cause the birth and death. Specifically, the birth value is equal to the length of the longest edge in the simplex $\sigma_{b_x}$ which \emph{encloses} some perimeter or void, while the death value is equal to the length of the longest edge in the simplex $\sigma_{d_x}$ which \emph{fills in} the enclosed space.

Further denote the unique longest edge of any $\sigma$ by the $1$-simplex $\bar{\sigma}$. That is, the generator $x=(b_x,d_x)$ has birth and death values determined by simplices $\sigma_{b_x},\sigma_{d_x}$, which have longest edges given by $\bar{\sigma}_{b_x}=\{p_{b_x}^1,p_{b_x}^2\}\subset \mathbb{P}$ and $\bar{\sigma}_{d_x}=\{p_{d_x}^1,p_{d_x}^2\}\subset \mathbb{P}$.

See Figure \ref{grad_pc_movement} for a visual explanation of the following lemma.

\begin{lemma}[Lemma 3.5 in \citep{continuation}]\label{continuation_lemma} 
Any generator $x=(b_x,d_x)\in\mathrm{PD}(\mathbb{P})$, seen as a map $\mathbb{R}^n\to\mathbb{R}^2$, is differentiable.
Namely, for $p\in \mathbb{P}$,
$$\dfrac{\partial b_x}{\partial p}=\dfrac{(\delta_{p,p_{b_x}^1}-\delta_{p,p_{b_x}^2})(p_{b_x}^1-p_{b_x}^2)}{\|p_{b_x}^1-p_{b_x}^2\|}\,\,\,\,\text{ and }\,\,\,\,\dfrac{\partial d_x}{\partial p}=\dfrac{(\delta_{p,p_{d_x}^1}-\delta_{p,p_{d_x}^2})(p_{d_x}^1-p_{d_x}^2)}{\|p_{d_x}^1-p_{d_x}^2\|}.$$
The Kronecker delta difference term simply gives us that the partial derivative is zero save for when $p$ is precisely one of the two determining points of the birth (resp. death) edge.
\end{lemma}



\begin{figure}[t]
    \centering
    \includegraphics[width=.65\textwidth]{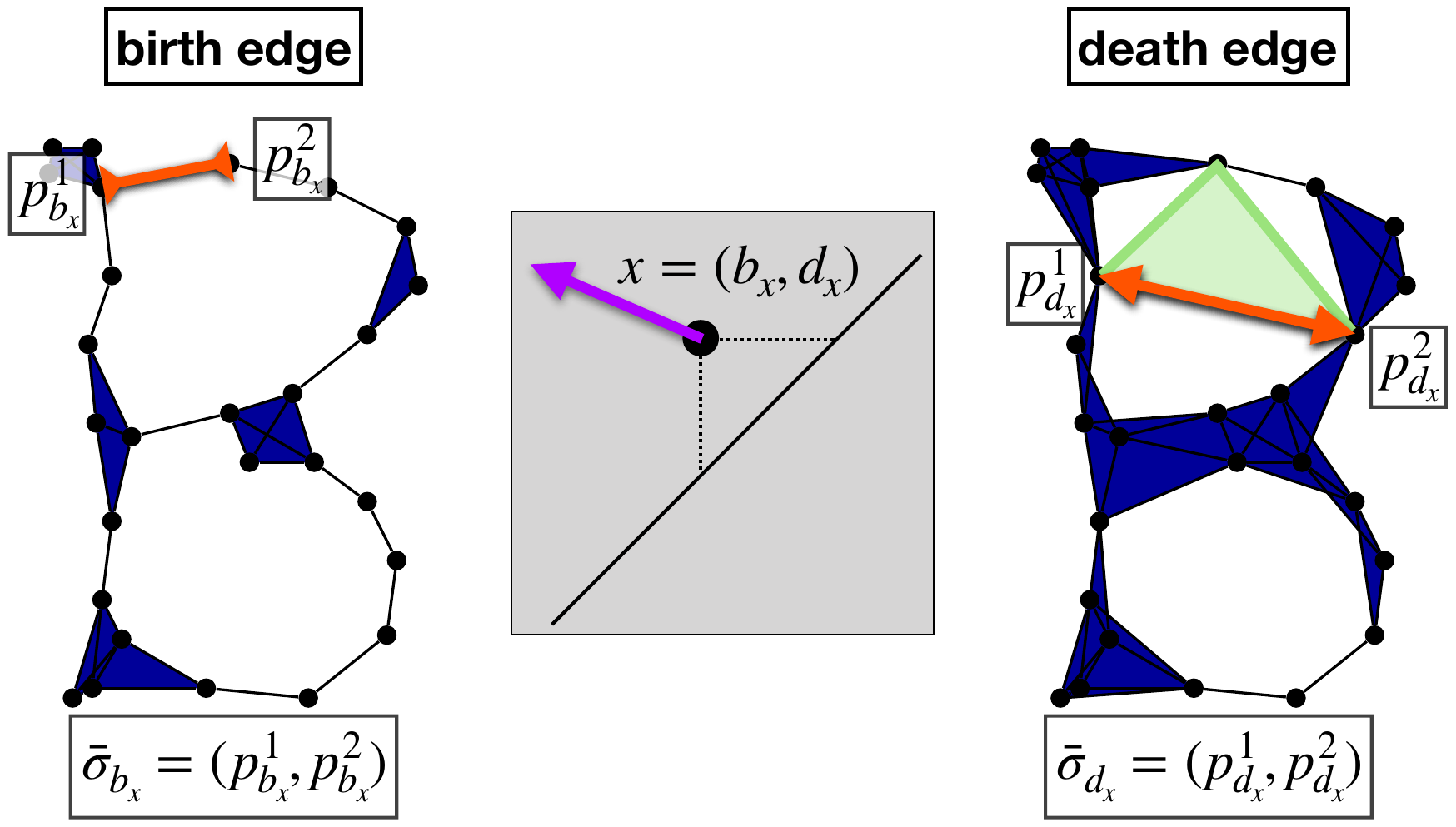}
    \caption{A visualization of how arrows (gradients) in the PD translate to arrows in the point cloud as prescribed by Lemma \ref{continuation_lemma}. A non-trivial generator in the PD corresponds to two edges: the one responsible for its birth (the edge that closes off some perimeter) and the one responsible for its death (the longest edge of the triangle that fills in that perimeter). Movement of the generator in the birth coordinate corresponds to \emph{extension} and \emph{contraction} of the birth edge, while movement in the death coordinate is the same for the death edge. By extension and contraction, we mean that the two points comprising this `edge' are moved simultaneously directly toward or directly away from each other. Specifically, movement of the generator in the negative direction of one axis represents contraction, while movement in the positive direction represents extension.}
    \label{grad_pc_movement}
\end{figure}

\subsection{Gradients for Topological Losses}\label{sec_gradients_top}
We consider the gradient of the map $\alpha \mapsto d(\alpha,\beta)$ when $d$ is either the squared metric $\FG$ introduced in \eqref{eq:FG_metric} or the Sinkhorn divergence for persistence diagrams $\SFG_\entropicregparam$ \eqref{eq:HUROT_SFG} based on the HUROT problem $\FG_\entropicregparam$ \eqref{eq:HUROT_TDA}. 
When minimizing the map $\alpha \mapsto \FG(\alpha,\beta)$, a typical reasoning used in topological data analysis literature
is the following: if we let $\zeta^*$ be the (generically unique) optimal partial matching between $\alpha$ and $\beta$,
it is natural to define a pseudo-gradient of $\alpha \mapsto \mathrm{FG}(\alpha,\beta)$ by 
\begin{equation}\label{eq:grad_FG}
    \nabla_{\alpha} \FG(\alpha, \beta) : x \mapsto \nabla_x \|x - \zeta^*(x)\|^2 = 2 (x - \zeta^*(x)).
\end{equation} 
That is, a gradient step pushes each $x$ in $\alpha$ toward its corresponding $\zeta^*(x)$ (whether this belongs to the support of $\beta$ or to the diagonal $\partial \Omega$). 
A pseudo-gradient of the map $\alpha \mapsto \FG(\alpha,\beta)$ defined in this way can be identified with the usual meaning of gradient for the map $(\RR^2)^N \to \RR$ given by $(x_1,\dots,x_N) \mapsto \FG\left( \sum_{i=1}^N \delta_{x_i}, \beta \right)$. 

A key contribution of \citet[Section~3.2 and 3.3]{leygonie2021framework} was in showing that this pseudo-gradient can indeed be used in a chain rule setting with \cref{continuation_lemma}. 
Formally, gradients of maps of the form $Q \mapsto L(\mathrm{PD}(Q))$ (with $L: \DD \to \RR$) can be obtained as the composition of the corresponding gradient and pseudo-gradient. 
In particular, though the parametrization of $\alpha$ and of any gradients by the $n$-tuple $(x_1,\dots,x_N)$ depends on the choice of an ordering, the eventual chain rule is independent of this choice.

Building on this result, we derive the gradient of the Sinkhorn divergence for persistence diagrams that we use in our implementation. To alleviate notation,  for some partial transport plan $P \in \Pi$ we will write its total mass as $$\mathbf{M}(P) = \sum_{\substack{1 \leq i \leq N, \\ 1 \leq j \leq M}} P_{ij}$$ and we let $$\mathbf{M}_i(P) = \sum_{1 \leq j \leq M} P_{ij}$$ (the fraction mass transported from $x_i$ to any of the $y_j$, $j = 1,\dots,M$). 
We then introduce the \emph{barycentric map} with parameter $\entropicregparam > 0$ from $\alpha$ to $\beta$ defined on $\{x_1,\dots,x_N\}$ by
    \[ \frac{1}{2} T^\entropicregparam_{\alpha,\beta}(x_i) = \projdiag(x_i) \left(1 - \mathbf{M}_i (P_{\alpha,\beta}^\entropicregparam)\right) + \sum_{j=1}^M (P^\entropicregparam_{\alpha,\beta})_{ij} \cdot y_j,\]
where $P^\entropicregparam_{\alpha,\beta}$ is the unique optimal partial transportation plan for $\FG_\entropicregparam(\alpha,\beta)$. 
Note that the self-optimal partial transportation plan $P^\entropicregparam_{\alpha,\alpha}$ is necessarily symmetric.

\begin{proposition}
    Let $\alpha = \sum_{i=1}^N \delta_{x_i},\ \beta = \sum_{j=1}^M \delta_{y_j}$ be two persistence diagrams. The gradient of $\alpha \mapsto \SFG_\entropicregparam(\alpha,\beta)$ in the sense of \citep{leygonie2021framework} is the map defined on $\{x_1,\dots,x_N\}$ by
    \begin{equation}\label{eq:grad_SFG}
        \nabla_\alpha \SFG_\entropicregparam(\alpha,\beta) : x_i \mapsto T^\entropicregparam_{\alpha,\alpha}(x_i) - T^\entropicregparam_{\alpha,\beta}(x_i) + \entropicregparam K_i
    \end{equation}
    with
    \[ K_i = \left[ \frac{1}{2} \frac{d_i - b_i}{\pers(\alpha)} \left(\mathbf{M} (P^{\entropicregparam}_{\alpha,\beta})  - \mathbf{M}(P^{\entropicregparam}_{\alpha,\alpha}) \right) - \frac{2}{d_i - b_i} \left( \mathbf{M}_i (P^{\entropicregparam}_{\alpha,\beta}) - \mathbf{M}_i (P^{\entropicregparam}_{\alpha,\alpha}) \right) \right] \begin{pmatrix}-1\\1\end{pmatrix}, \]
    where $x_i = (b_i, d_i) \in \Omega$ (so that $d_i > b_i$), for $i = 1,\dots,N$. 
\end{proposition}

\begin{proof}
The gradient with respect to $x_i$ of $\SFG_\entropicregparam$ is given by
\begin{align*}
    &\nabla_{x_i} \left[ \sum_{{i'},j} \|x_{i'} - y_j\|_2^2 (P^\entropicregparam_{\alpha,\beta})_{{i'}j} + \sum_{i'=1}^N \|x_{i'} - \projdiag(x_{i'})\|_2^2 \left( 1 - \sum_{j=1}^M (P^\entropicregparam_{\alpha,\beta})_{{i'}j} \right) \right] \\
    - \frac{1}{2} &\nabla_{x_i} \left[ \sum_{{i'}j} \|x_{i'} - x_j\|_2^2 (P^\entropicregparam_{\alpha,\alpha})_{i' j} + \sum_{i'} \|x_{i'} - \projdiag(x_{i'})\|_2^2 \left(1 - \sum_{j=1}^M (P^\entropicregparam_{\alpha,\beta})_{{i'}j} \right) \right] \\
    + &\entropicregparam \nabla_{x_i} \varphi(x_1,\dots,x_N),
\end{align*}
where
\[\varphi(x) = \frac{1}{2} \left(\KL\left(P^\entropicregparam_{\alpha,\beta} \bigg| \frac{\mathbf{a}\mathbf{b}^T}{\pers(\alpha)} \right) + \KL\left(P^\entropicregparam_{\alpha,\beta} \bigg| \frac{\mathbf{a}\mathbf{b}^T}{\pers(\beta)}\right) \right) - \frac{1}{2} \KL\left(P^\entropicregparam_{\alpha,\alpha} \bigg| \frac{\mathbf{a}\mathbf{a}^T}{\pers(\alpha)} \right) \]
accounts for the entropic regularization terms. 
Note that $\varphi$ depends on $x$ since the column vector $\mathbf{a}$ accounts for the squared distance between the points in $\spt(\alpha)$ and their projection on the diagonal $\thediag$.

Computing the gradient of the first term yields
\begin{align*}
    &\nabla_{x_i} \left[ \sum_{i',j} \|x_{i'} - y_j\|_2^2 (P^\entropicregparam_{\alpha,\beta})_{{i'}j} + \sum_{i' =1}^N \|x_{i'} - \projdiag(x_{i'})\|_2^2 \left(1 - \sum_{j=1}^m (P^\entropicregparam_{\alpha,\beta})_{{i'}j} \right) \right] \\
    = &2 \sum_{j=1}^M (x_i - y_j) (P^\entropicregparam_{\alpha,\beta})_{ij} + 2 (x_i - \projdiag(x_i)) \left(1 - \sum_{j=1}^M (P^\entropicregparam_{\alpha,\beta})_{ij} \right) \\
    = &2x_i - 2 \left(\sum_{j=1}^M y_j (P^\entropicregparam_{\alpha,\beta})_{ij} + \projdiag(x_i) \left(1 - \sum_{j=1}^M (P^\entropicregparam_{\alpha,\beta})_{ij}\right) \right) \\
    = &2x_i - T^\entropicregparam_{\alpha,\beta}(x_i).
\end{align*}
For the second term, using the symmetry of $P^\entropicregparam_{\alpha,\alpha}$ and the above calculation, we have
\begin{align*}
    &\frac{1}{2} \nabla_{x_i} \left[ \sum_{i',j} \|x_{i'} - x_j\|_2^2 (P^\entropicregparam_{\alpha,\alpha})_{{i'}j} + \sum_{i' =1}^N \|x_{i'} - \projdiag(x_{i'})\|_2^2 \left(1 - \sum_{j=1}^M (P^\entropicregparam_{\alpha,\alpha})_{{i'}j} \right) \right] \\
    = &2x_i - T^\entropicregparam_{\alpha,\alpha}(x_i).
\end{align*}
The gradient of $\varphi$ is obtained by expanding the factors in the different $\KL$ terms and observing that $\nabla_{x_i} \|x_i - \projdiag(x_i) \|^2 = (d_i - b_i) \begin{pmatrix}-1 \\ 1 \end{pmatrix}$, resulting in 
\begin{align*}
    &\frac{1}{2} \nabla_{x_i} \log( \pers(\alpha) ) \sum_{i'j} (P^\entropicregparam_{\alpha,\beta})_{i'j}  - \nabla_{x_i} \sum_{i' j} (P^\entropicregparam_{\alpha,\beta})_{i' j}  \log( \|x_j - \projdiag(x_j)\|^2 ) \\
    &- \frac{1}{2} \nabla_{x_i} \log(\pers(\alpha)) \sum_{jk} (P^\entropicregparam_{\alpha,\alpha})_{jk} + \nabla_{x_i} \sum_{jk} (P^\entropicregparam_{\alpha,\alpha})_{jk} \log(\|x_j - \projdiag(x_j)\|^2)  \\
    &+ \nabla_{x_i} \frac{1}{2} \pers(\alpha) - \nabla_{x_i} \frac{1}{2} \pers(\alpha)) \\
    = &\left[ \frac{1}{2} \frac{d_i - b_i}{\pers(\alpha)} \left(\mathbf{M} (P^{\entropicregparam}_{\alpha,\beta})  - \mathbf{M}(P^{\entropicregparam}_{\alpha,\alpha}) \right) - \frac{2}{d_i - b_i} \left( \mathbf{M}_i (P^{\entropicregparam}_{\alpha,\beta}) - \mathbf{M}_i (P^{\entropicregparam}_{\alpha,\alpha}) \right) \right] \begin{pmatrix}-1\\1\end{pmatrix}.
\end{align*}
Putting these terms together completes the proof.
\end{proof}

\begin{remark}
Interestingly, the term $T^\entropicregparam_{\alpha,\beta} - T^\entropicregparam_{\alpha,\alpha}$ appearing in our gradient can be related to \cite[Section~4.2]{carlier2022lipschitz}, where it is shown that this quantity represents the tangent vector field for the gradient flow of the Sinkhorn divergence between probability measures (instead of persistence diagrams).
Because of the particular role played by the diagonal in our problem\footnote{Formally, it induces a \emph{spatially varying} divergence \cite[Section~2.4]{sejourne2019sinkhorn}}, we have the additional term $\entropicregparam K_i$ which acts orthogonally to the diagonal to account for our distance-weighted entropic term. 
In particular, when we take the limit as $\entropicregparam \to 0$, since $T^\entropicregparam_{\alpha,\beta} \to P_{\alpha,\beta}$ and $T^\entropicregparam_{\alpha,\alpha} \to \mathrm{id}$, we expect to retrieve the descent direction $P_{\alpha,\beta} - \mathrm{id}$ which exactly describes the pseudo-gradient \eqref{eq:grad_FG} and is the geodesic between $\alpha$ and $\beta$ for the usual $\FG(\cdot,\cdot)^{\frac{1}{2}}$ metric \citep{turner2014frechet}. This is analogous to the displacement interpolation proposed by McCann in OT literature \citep{mccann1997convexity}. 
\end{remark}

\section{Topological Node2vec}\label{sec_tn2v}
In order to utilize topological information in \ntv{}, we need to be able to extract a persistence diagram from a graph. Recall from Section \ref{sec_ph} that the Vietoris-Rips construction of the persistence diagram relies only on pairwise distances between points, for which we can easily make an analog involving the weights of edges in a graph $G = (V, E, w)$. Namely, the persistence diagram $\mathrm{PD}(G)$ of a graph $G$ is the one extracted from the Vietoris-Rips filtration $\{\mathrm{VR}_\vrparameter\}$ where each edge $(u, v) \in V \times V$ is inserted at time \[\frac{1}{(w(u,v)+\gamma)^\nu}\] for $\nu,\gamma >0$ with $\gamma$ very small ($1/\gamma^\nu$ becoming the distance between points connected by zero-weight edges in the graph). This ensures that vertices $u,v$ connected by a strong (resp. weak) edge in the graph $w(u,v)$ are regarded as close (resp. far) in the view of the Vietoris-Rips filtration.

\begin{definition}
    For a \ntv{} model with input graph $G$ and parameters $\Theta = (W_1, W_2)$, we define the topological loss term 
    \begin{equation}\label{topological loss function}
        L_1(G,\Theta)=\SFG_\entropicregparam(\mathrm{PD}(\mathrm{Emb}(\Theta)),\mathrm{PD}(G)). 
    \end{equation}
    Recall from Definition \ref{def_structure} that $\mathrm{Emb}(\Theta)$ is the matrix $W_1$ with its rows viewed as Euclidean points. 
    It follows immediately that the gradient of $L_1$ with respect to any coordinate of $W_2$ is zero.
\end{definition}

We are now ready to state the gradient of the topological loss function $L_1(\Theta, G)$ with respect to $\Theta$, or, as clarified above, with respect to the points $p$ in the current embedding $\mathrm{Emb}(\Theta)$.

\begin{theorem}[Gradient of Topological Loss Term]\label{theorem_main}
Let $\alpha = \mathrm{PD}(\mathrm{Emb}(\Theta))$ and $\beta \subset \Omega$.
Let $\zeta_\entropicregparam$ be the gradient of $\alpha \mapsto \SFG_\entropicregparam(\alpha,\beta)$ 
as provided by \eqref{eq:grad_SFG}. 
Let $p\in\mathrm{Emb}(\Theta)$. 
We have
\begin{align*}\frac{\partial L_1}{\partial p}=\sum_{x\in\mathrm{PD}(\mathrm{Emb}(\Theta))}\zeta_\entropicregparam(x)\cdot\dfrac{\partial x}{\partial p},
\end{align*}
where the partial derivative $\dfrac{\partial x}{\partial p}$ is provided by \cref{continuation_lemma}.
\end{theorem}

\begin{proof} From the chain rule of \citep{leygonie2021framework} (see \cref{sec_gradients_top}) and \cref{eq:grad_SFG} it follows that
\begin{align*} \dfrac{\partial L_1}{\partial p} &= \sum_{x \in \mathrm{PD}(\mathrm{Emb}(\Theta))} \dfrac{\partial \SFG_\entropicregparam}{\partial x} \cdot \dfrac{\partial x}{\partial p} \\
&= \sum_{x \in \mathrm{PD}(\mathrm{Emb}(\Theta))} \zeta_\entropicregparam(x) \cdot \dfrac{\partial x}{\partial p}.
\end{align*}
\end{proof}

We can now train \ntv{} using a loss
\[ L(\Theta) = \lambda_0 L_0(T, \Theta) + \lambda_1 L_1(G, \Theta) \]
for some parameters $\lambda_0,\lambda_1 \geq 0$, using a standard gradient descent framework with a learning rate $\eta > 0$. Let $\Theta(k)$ denote the parameters after $k \in \mathbb{Z}_{\geq 0}$ steps of gradient descent updates. Then the next epoch's parameters are given by
\begin{equation}\label{eq_new_gradient}
\Theta(k+1):=\Theta(k) - \eta\left(\lambda_0
\nabla_{\Theta(k)} L_0(T(k),\Theta(k))
+\lambda_1\nabla_{\Theta(k)}L_1(G,\Theta(k))\right)
\end{equation}
with individual gradients given by Proposition \ref{prop_ce_gradient} and Theorem \ref{theorem_main}. We summarize this in Algorithm \ref{alg_topnode2vec}. 

\begin{algorithm}
\caption{Topological Node2vec algorithm}\label{alg_topnode2vec}
\begin{algorithmic}
    \INPUT a weighted graph $G=(V,E,w)$, learning rates $\eta, \lambda_0,\lambda_1$, embedding dimension $m$, neighborhood parameters $(l,r,p,q)$, entropic regularization parameter $\entropicregparam > 0$;
    \OUTPUT $\mathbb{P}\subset\mathbb{R}^m$ with $|\mathbb{P}|=|V|$
    \start
    \State • initialize the parameter matrices $\Theta(0)$; let $k=0$;
    \State • compute the persistence diagram $\mathrm{PD}(G)$
    \While{not converged}
    \State • compute the persistence diagram $\mathrm{PD}(\mathrm{Emb}(\Theta(k)))$
    \State • for each node $v\in V$, generate the \emph{current neighborhood vector} $C_v$ and the \emph{training neighborhood vector} $T_v$; let $C(k) = \{C_v(k)\}_{v \in V}$ and $T(k):=\{T_v(k)\}_{v\in V}$
    \State • update $\Theta(k+1)$ using \eqref{eq_new_gradient}
    \State • $k\mathrel{+}=1$
    \EndWhile
\end{algorithmic}
\end{algorithm}

\section{Numerical Experiments}\label{sec_experiments}

In this section, we elaborate on two synthetic experiments designed to demonstrate what we gain over the base \ntv{} model by including our topological loss function \eqref{topological loss function}. 

Our code is publicly available at this \href{https://github.com/killianfmeehan/topological_node2vec}{repository}. 
Its implementation provides both CPU and GPU backend. The CPU-backend relies on \texttt{Gudhi} \citep{maria2014gudhi}, while the GPU-backend is based on a fork of \texttt{Ripser++} \citep{bauer2021ripser,zhang2020gpu} where we adapted the code in order to access the correspondence between a generator $x$ in the PD and the points in the point cloud responsible for the birth edge and death edge.
For an in-depth examination of all the hyperparameters of this network (both those related to original \ntv{} as well as our topological additions), please see the readme file and examples notebook in the 
repository.

\subsection{Experiment: $S^1$ Made of 8 Smaller $S^1$s}
\begin{figure}
    \centering
    \includegraphics[width=0.25\textwidth]{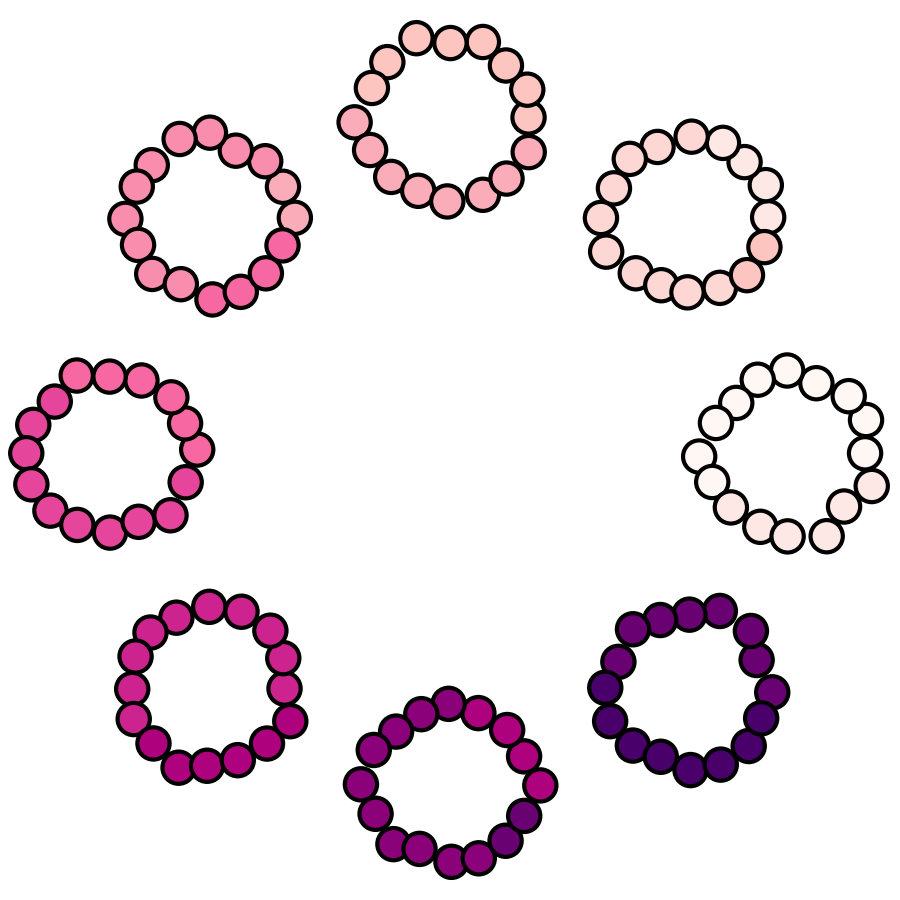}
    \caption{Euclidean data for the synthetic experiment, from which we derive a weighted adjacency matrix to input into (Topological) Node2vec. This data set demonstrates overt topological features on two scales: the eight smaller cycles and the larger emergent cycle formed by their arrangement.}
    \label{fig_s1x8_original}
\end{figure}

We look at a dataset consisting of 8 circles each with 16 points arranged radially in a larger circle, seen in Figure \ref{fig_s1x8_original}. While the points and the sampled circles themselves are distributed uniformly, each point is wiggled by some small random noise, ensuring that the input data satisfies Vietoris-Rips general position (which is guaranteed if all pairwise distances are unique).

\textbf{Node2vec neighborhood generation.}
We first demonstrate that, as shown in Figure \ref{fig_nbhd_params}, for this data set it is best to forego the random neighborhood generation of \ntv{} and simply use as neighborhoods the columns of the adjacency matrix corresponding to each vertex. (See Remark \ref{rmk_r_infinity}.)


\begin{figure}[t]
     \centering
     \begin{subfigure}[t]{0.17\textwidth}
         \centering
         \includegraphics[width=\textwidth]{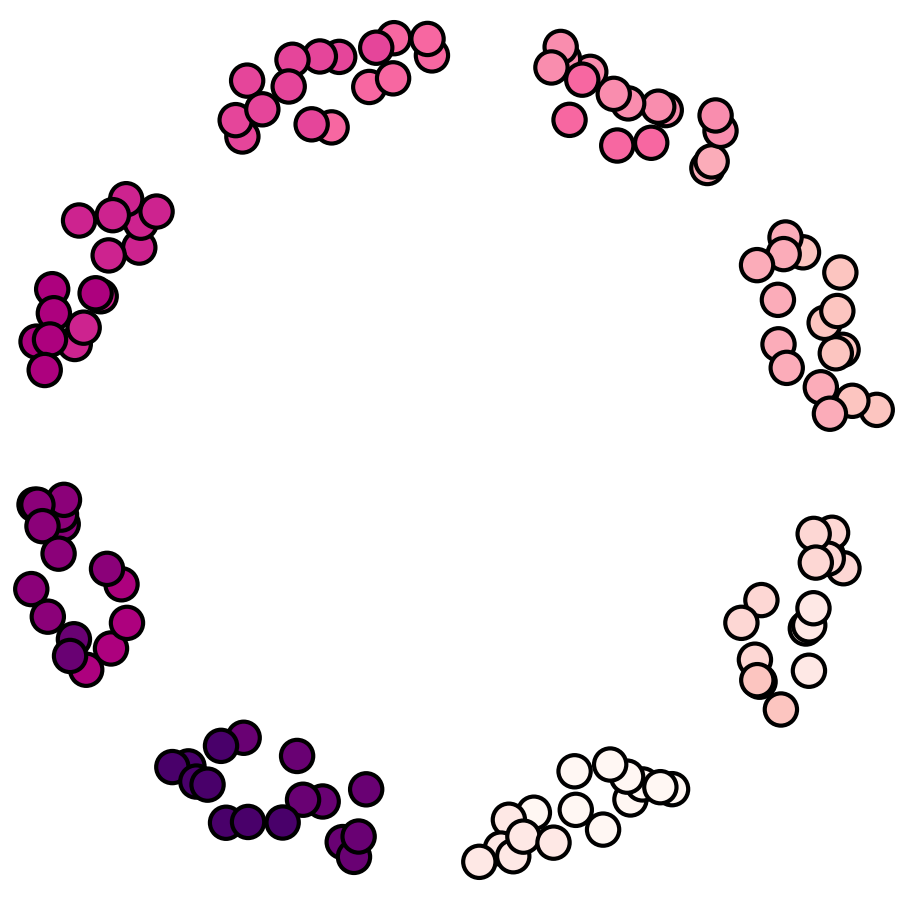}
         \caption{$l=1, r=8$}
     \end{subfigure}
     \hfill
     \begin{subfigure}[t]{0.17\textwidth}
         \centering
         \includegraphics[width=\textwidth]{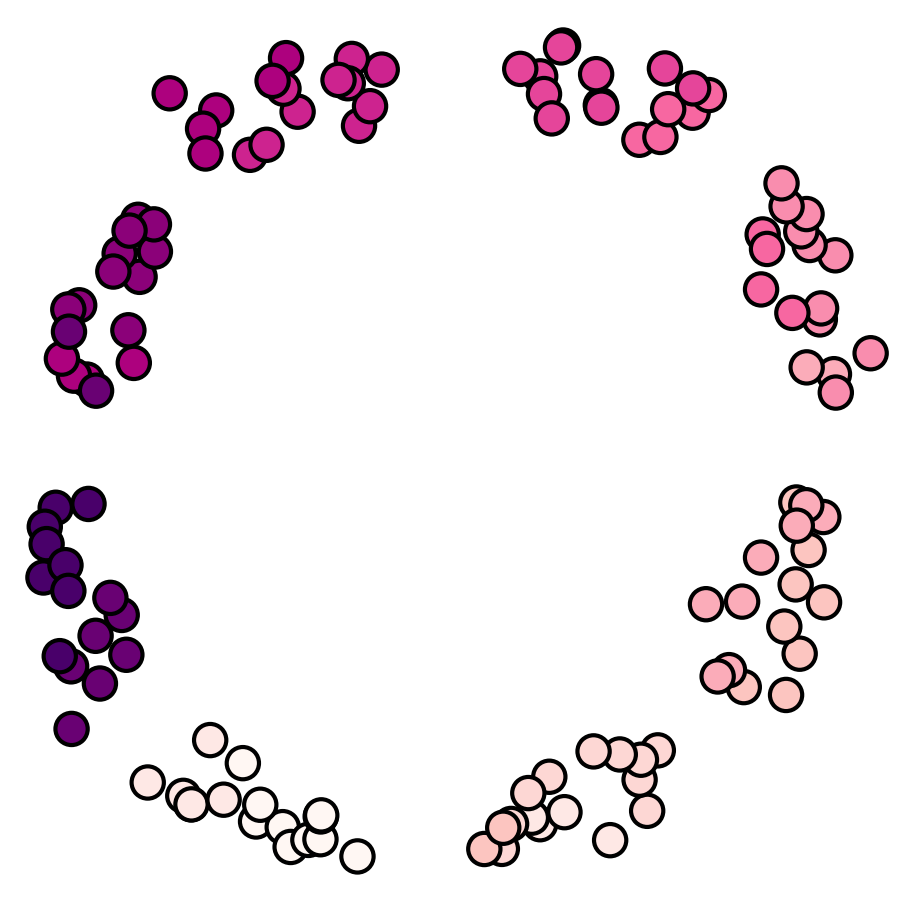}
         \caption{$l=2, r=4$}
     \end{subfigure}
     \hfill
     \begin{subfigure}[t]{0.17\textwidth}
         \centering
         \includegraphics[width=\textwidth]{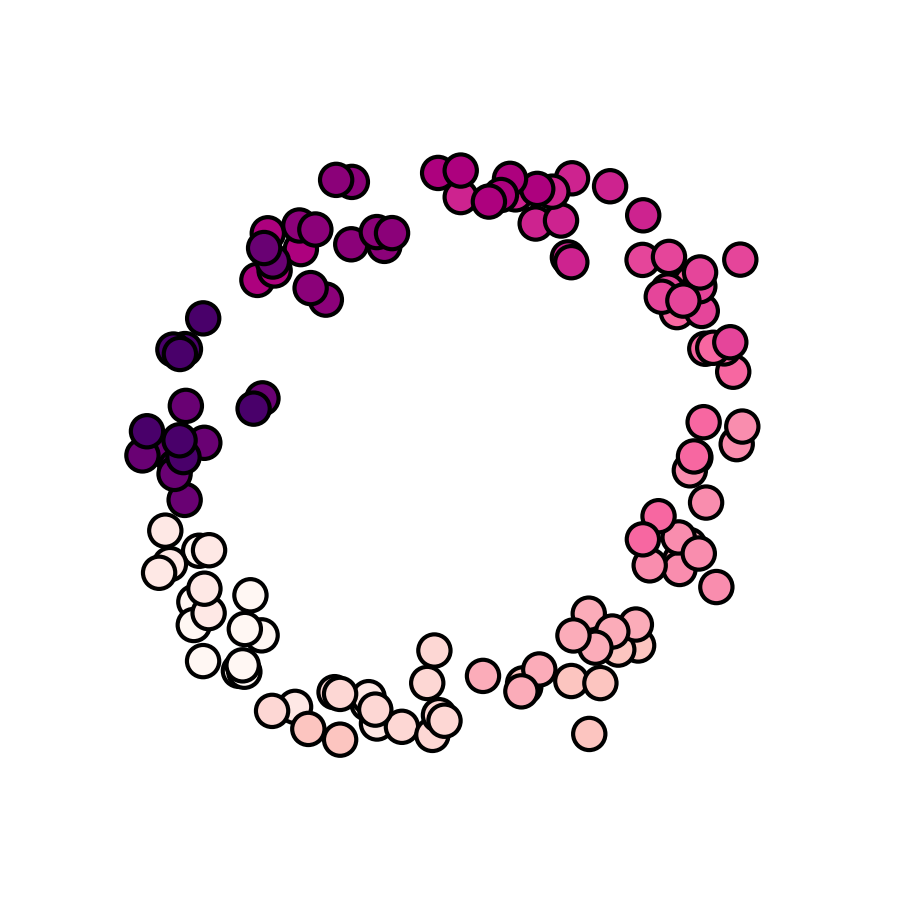}
         \caption{linear interpolation from $l=1,r=8$ to $l=8,r=1$}
     \end{subfigure}
     \hfill
     \begin{subfigure}[t]{0.17\textwidth}
         \centering
         \includegraphics[width=\textwidth]{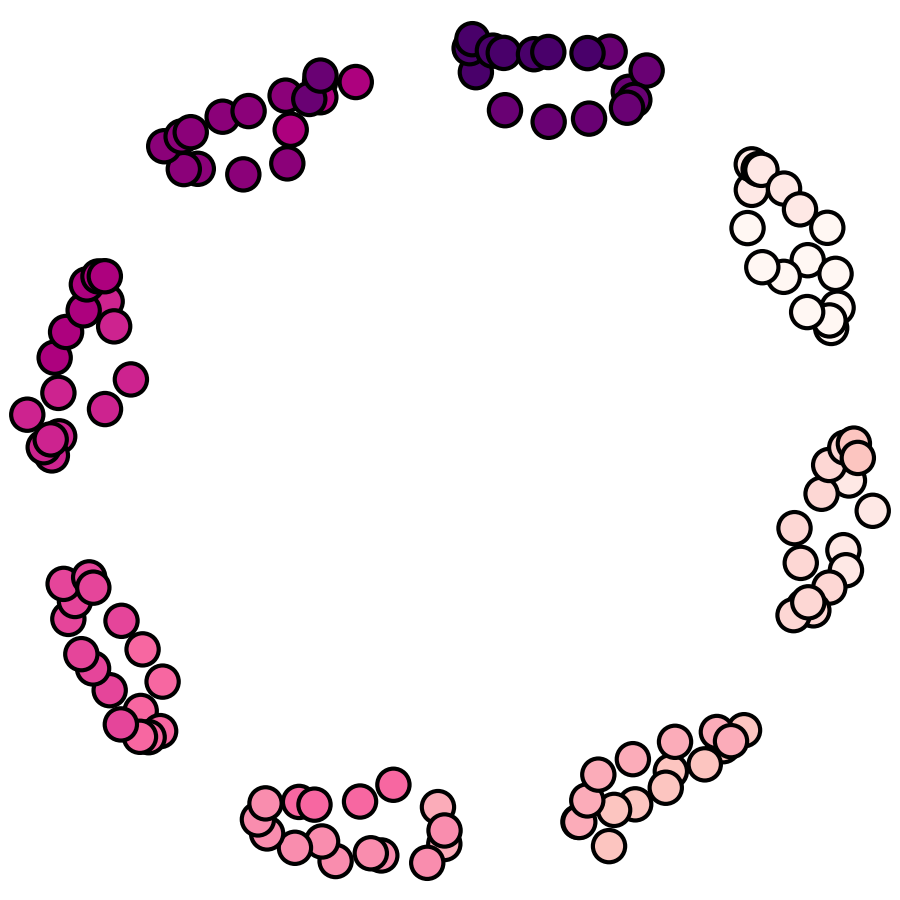}
         \caption{$l=1, r=16$}
     \end{subfigure}
     \hfill
     \begin{subfigure}[t]{0.17\textwidth}
         \centering
         \includegraphics[width=\textwidth]{images/nbhd_params_full.png}
         \caption{$l=1, r=\infty$, as in Remark \ref{rmk_r_infinity}}
     \end{subfigure}
    \caption{Various \emph{non-topological} \texttt{Node2vec} embeddings of the point cloud in Figure \ref{fig_s1x8_original} (specifically the reciprocal of its pairwise distance matrix) trained with different neighborhood parameters. Even in the best case, original \ntv{} dramatically erases topological information. (a\,--\,d) We see that $l=1$ and $r$ large give the best results. (e) This embedding was done using full-column vectors for neighborhood information.}
    \label{fig_nbhd_params}
\end{figure}

\textbf{Minibatches are important outside of computation time.} Upon introduction of the topological loss function \eqref{topological loss function}, we immediately encounter a novel problem. When identifying some topological feature in the embedding and matching it to a \emph{smaller} feature in the original data (something that is closer to the diagonal), the gradient update of the topological loss function \eqref{topological loss function} \emph{expands} the birth edge and \emph{shrinks} the death edge of the feature. (Recall Figure \ref{grad_pc_movement}.) However, the points that dictated the birth and death values of this feature are now all but guaranteed to be the same determining points in the next step of the network, causing a repeat of the same movement on the same points, leaving the rest of the data set untouched and the embedding progressively distorted. Taking properly sized \emph{minibatches} (subsampling the data at each step of the network) can completely remove this problem, as the points determining the birth or death of such a generator are likely to change from one step to the next. \cref{fig_minibatch} demonstrates this in detail. 

\begin{figure}[t]
     \centering
     \begin{subfigure}[t]{0.17\textwidth}
         \centering
         \includegraphics[width=\textwidth]{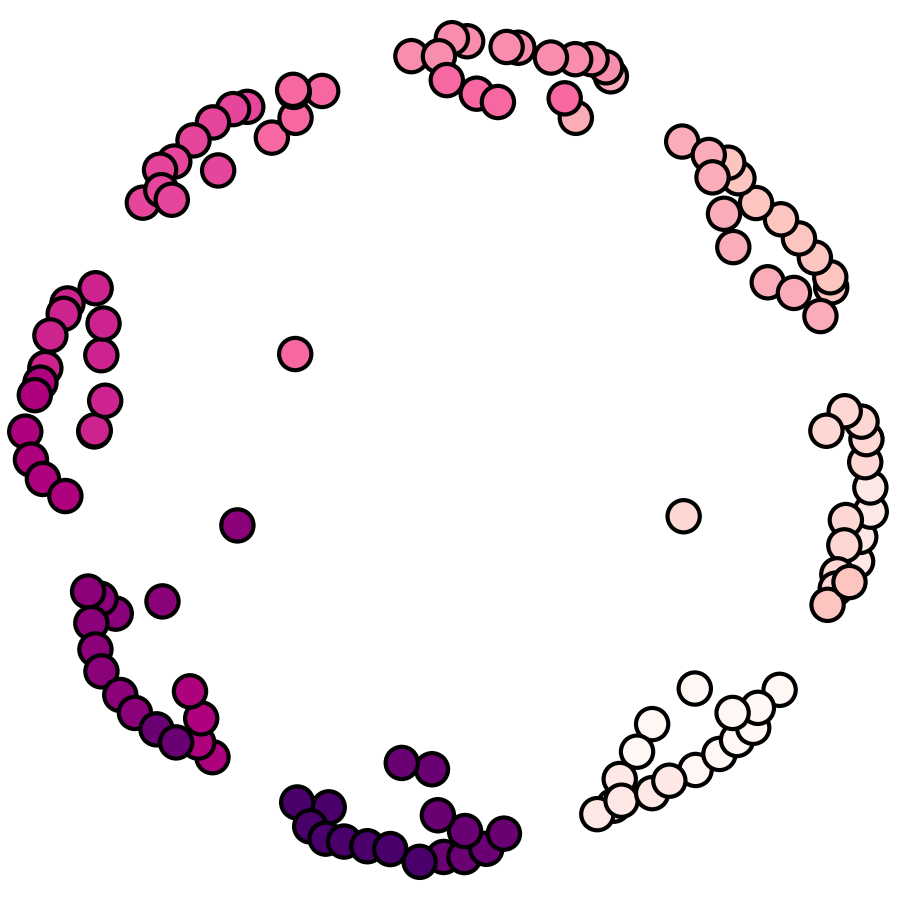}
         \caption{No minibatching: 100\% of data used at every epoch.}
     \end{subfigure}
     \hfill
     \begin{subfigure}[t]{0.17\textwidth}
         \centering
         \includegraphics[width=\textwidth]{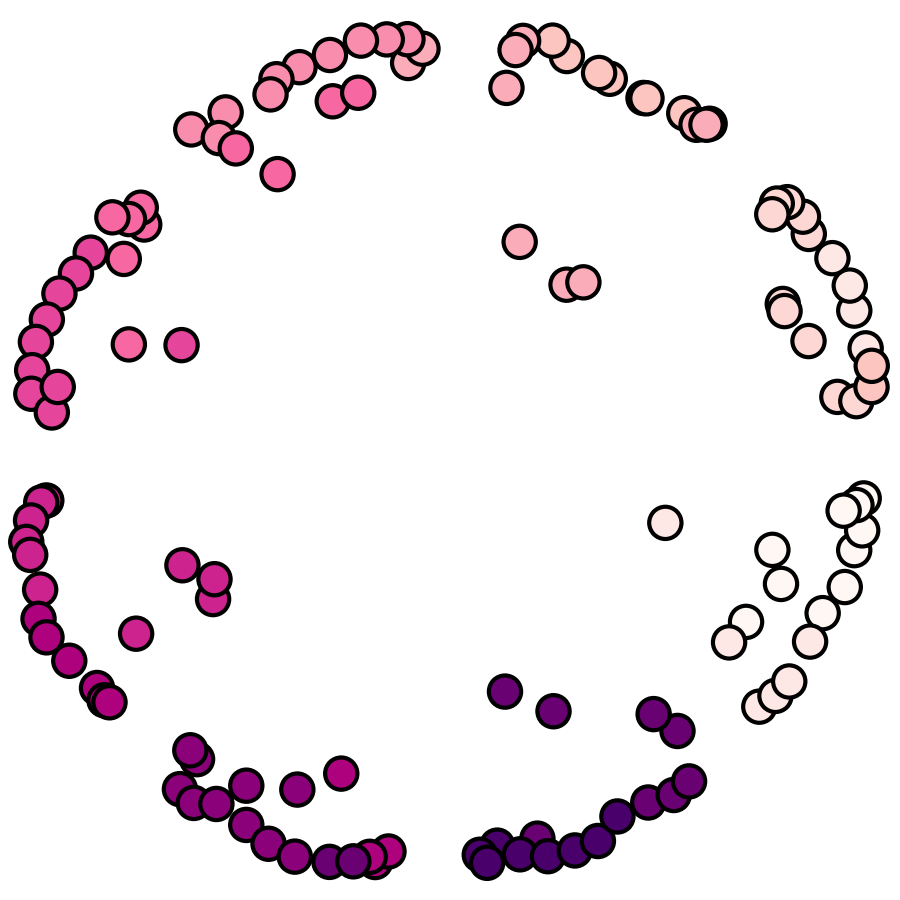}
         \caption{Minibatches of 75\% of data.}
     \end{subfigure}
     \hfill
     \begin{subfigure}[t]{0.17\textwidth}
         \centering
         \includegraphics[width=\textwidth]{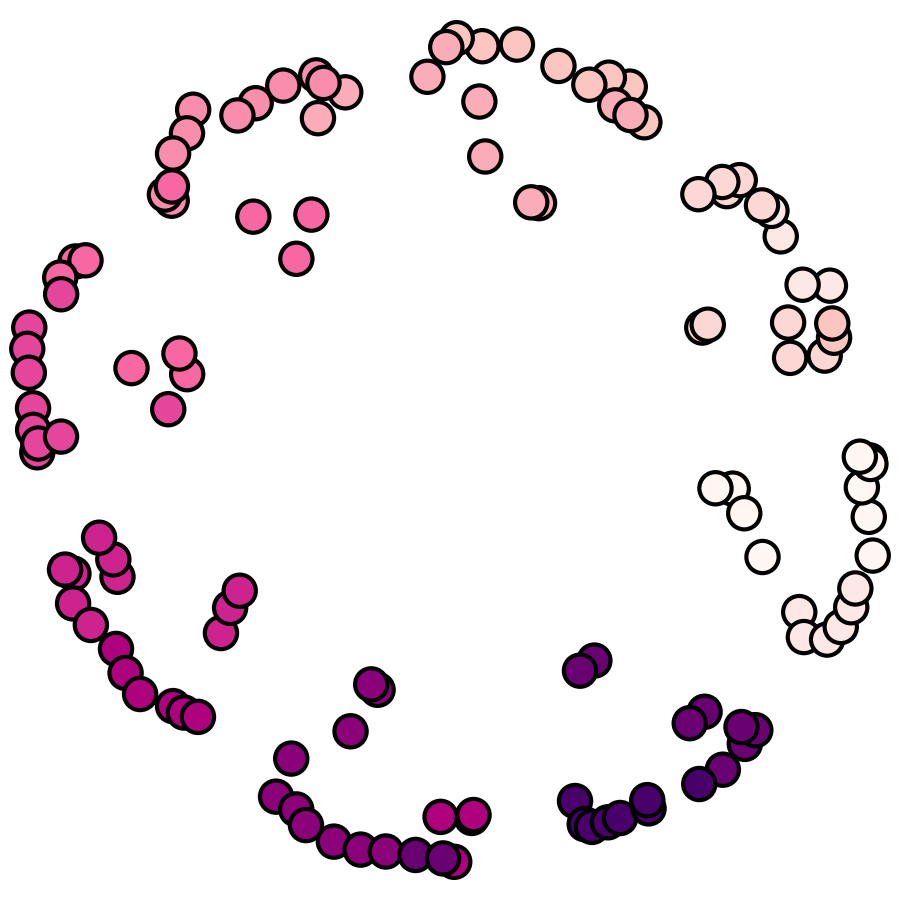}
         \caption{Minibatches of 50\% of data.}
     \end{subfigure}
     \hfill
     \begin{subfigure}[t]{0.17\textwidth}
         \centering
         \includegraphics[width=\textwidth]{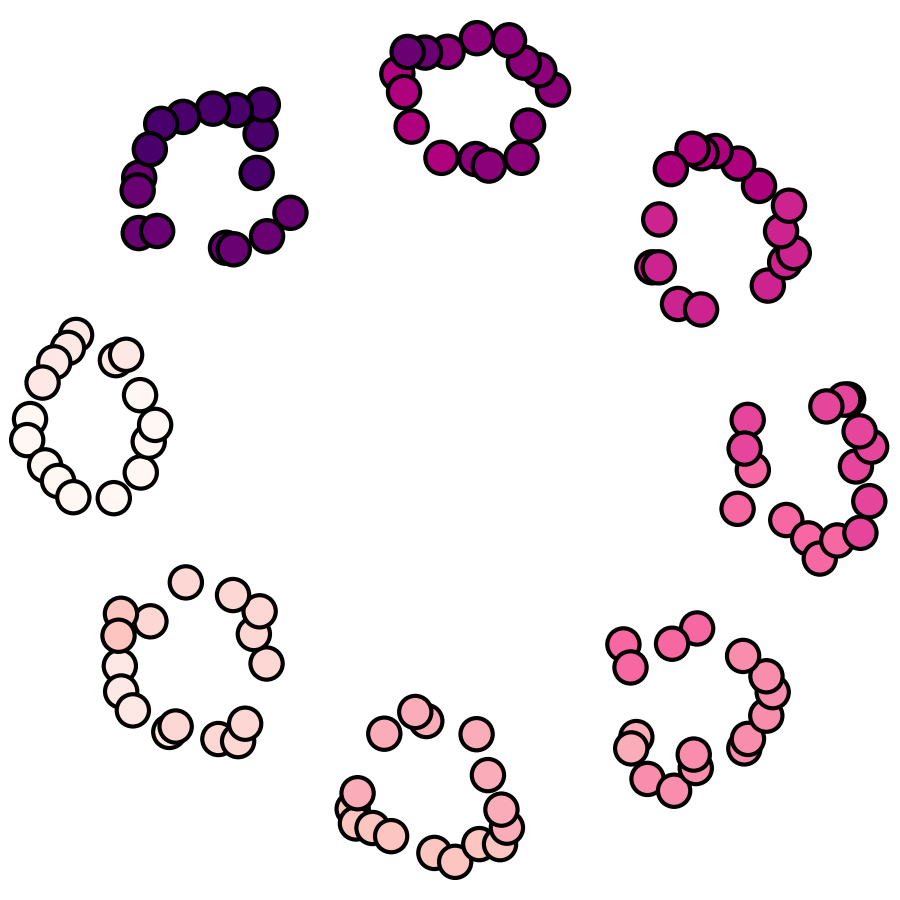}
         \caption{Minibatches of 25\% of data.}
     \end{subfigure}
     \hfill
     \begin{subfigure}[t]{0.17\textwidth}
         \centering
         \includegraphics[width=\textwidth]{images/s1x8_mbs_testing_graded.png}
         \caption{Minibatches graded from 25\% of data at the beginning of the embedding process to 100\% at the end.}
     \end{subfigure}
    \caption{\textbf{\emph{Topological Node2vec}} embeddings demonstrating the topological loss function's interaction with \emph{minibatch} size. (a\,--\,d) As minibatches decrease in size, the inward distortion of the points that provide the death value of the inner topological feature also decreases. (e) Testing for potential improvement, we performed another embedding beginning with 25\% minibatches and then linearly increasing to 100 \% as the model continued to train. All other parameters are fixed across these five experiments.}
    \label{fig_minibatch}
\end{figure}

\textbf{Results.}
Comparing the best results from original \ntv{} in Figure \ref{fig_nbhd_params} (e) and the best results from \emph{Topological Node2vec} in Figure \ref{fig_minibatch} (e), we can starkly see the information loss in base \ntv{} as well as our recovery of it via the topological loss function \eqref{topological loss function}.

\begin{figure}[t]
\captionsetup[subfigure]{justification=centering}
     \centering
     \begin{subfigure}[t]{0.30\textwidth}
         \centering
         \includegraphics[width=\textwidth]{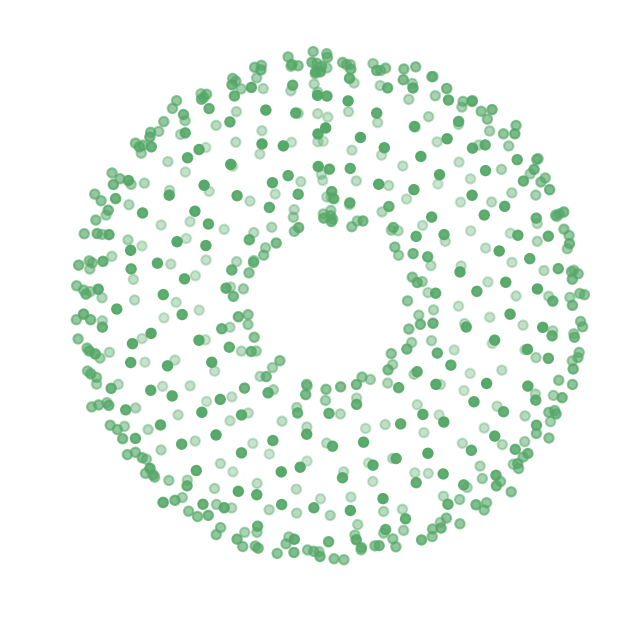}
         \caption{}
     \end{subfigure}
     \begin{subfigure}[t]{0.30\textwidth}
         \centering
         \includegraphics[width=\textwidth]{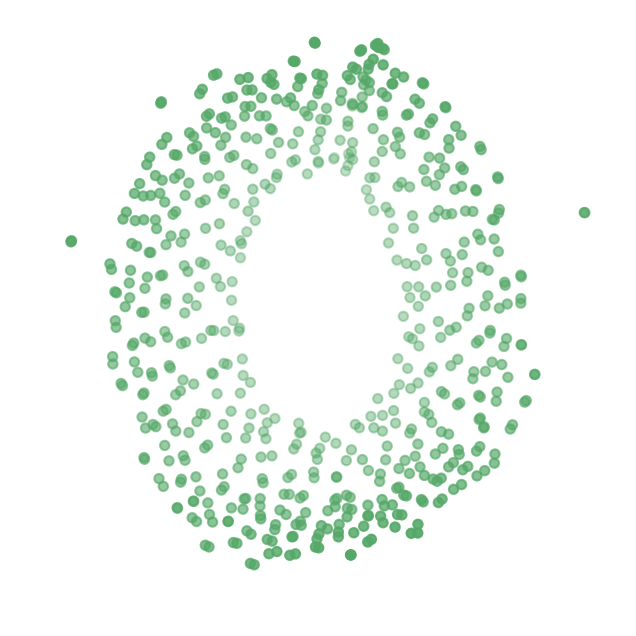}
         \caption{}
     \end{subfigure}
     \begin{subfigure}[t]{0.30\textwidth}
         \centering
         \includegraphics[width=\textwidth]{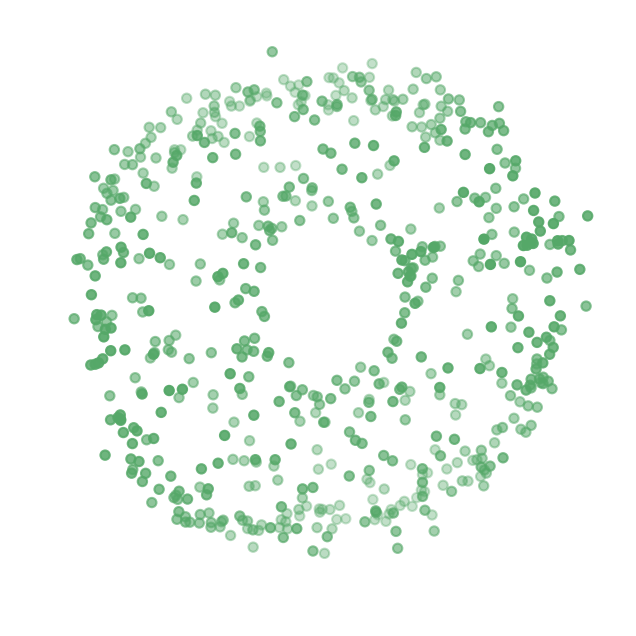}
         \caption{}
     \end{subfigure}
     \begin{subfigure}[t]{0.30\textwidth}
         \centering
         \includegraphics[width=\textwidth]{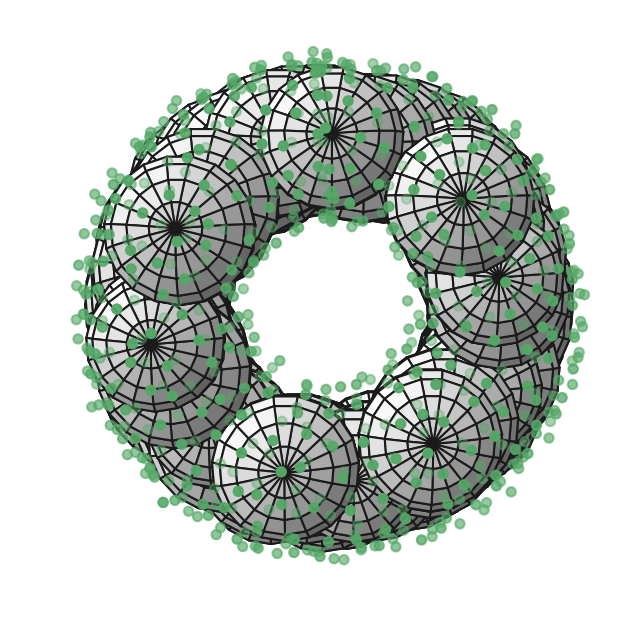}
         \caption{}
     \end{subfigure}
     \begin{subfigure}[t]{0.30\textwidth}
         \centering
         \includegraphics[width=\textwidth]{images/n2v_with.png}
         \caption{}
     \end{subfigure}
     \begin{subfigure}[t]{0.30\textwidth}
         \centering
         \includegraphics[width=\textwidth]{images/top_with.png}
         \caption{}
     \end{subfigure}
     \begin{subfigure}[t]{0.825\textwidth}
         \centering
         \includegraphics[width=\textwidth]{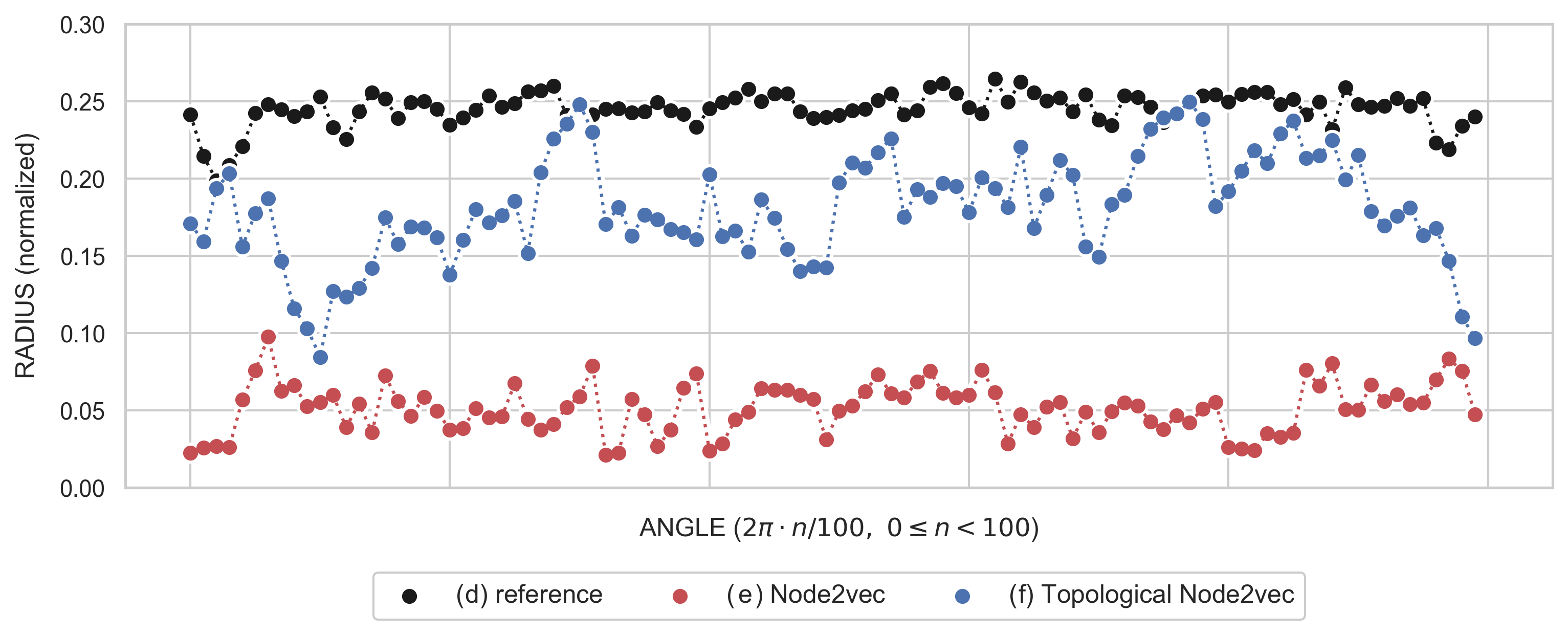}
         \caption{}
     \end{subfigure}
    \caption{3D render of various point clouds with largest inscribable spheres drawn inside. (a,d) Synthetic data created for this experiment. (b,e) Best embedding proposed by \ntv{} over various hyper-parameters. (c,f) Best embedding proposed by \emph{Topological Node2vec} over extensive space of hyper-parameters.  (g) A plot of the radii of all largest inscribable spheres drawn inside the preceding tori. In these examples, one hundred evenly distributed angles are taken, and all radii have been normalized by the final diameter of their corresponding embedding.}
    \label{fig_torus_experiments}
\end{figure}

\subsection{Experiment: Torus}\label{subsection_torus}
We sample a torus by arranging points uniformly over the surface. (These are not uniformly \emph{sampled} from a distribution, but rather a constructed covering. See the demonstration notebook in the \href{https://github.com/killianfmeehan/topological_node2vec}{repository} for the precise details.) Once again, in order to ensure Vietoris-Rips general position, we perturb all points by a small random vector.

The smooth torus has two independent one-dimensional topological features, representing one horizontal loop around the torus and one vertical one. The torus also has a single two-dimensional feature that represents the hollow interior. We benefit here from a very small minibatch size, eventually settling on 6.25\% of the full data set, constant across all epochs.

\textbf{Visualization.}
To visualize our embedded tori, we project the results into the first two principal components with the third principal component corresponding to distance from the viewer.

Visualizing the reconstruction of the two-dimensional topological feature (the hollow tube) is vital to judging the success of any torus embedding. To draw the `largest' inscribable spheres as in Figure \ref{fig_torus_experiments}, in increments of very small angles we take rectangular slices of the torus widened from the ray that starts at the torus' center and extends in the direction of the current angle. Averaging all points on this rectangle, we obtain a central point $c$. We then inscribe the largest possible sphere at that center.

\textbf{Results.}
We note the dramatic loss of topology in \ntv{} and equally dramatic recovery of such features with the inclusion of the topological loss function \eqref{topological loss function} demonstrated in Figure \ref{fig_torus_experiments}.

\section{Discussion}
The most obvious limiting factor for applying this model is the computation time of acquiring the persistence diagram of the proposed embedding at every epoch. Fortunately, in our synthetic examples, it is almost always the case that the best embeddings come out of very small minibatch sizes, which also result in dramatically reduced computation times.
As PD computation becomes further optimized, the model will benefit tremendously.

Since the generating point cloud (and thus the resulting persistence diagram) is only very marginally altered from one epoch to the next, this project seems potentially suitable for application of the work in \citep{vineyards} which may dramatically reduce computation time after generating the initial persistence diagram. However, this conflicts (at least at first glance) with subsampling minibatches at every epoch, which causes us to have largely unrelated PDs from one epoch to the next.

Part of the motivation for this project is the need for improved analysis of chromatin conformation capture data obtained from methods such as Hi-C and Pore-C \citep{ulahannan2019nanopore}.
These data represent the frequency of all two-contacts (Hi-C) or multi-contacts (Pore-C) among DNA segments and can be considered as graphs (Hi-C) or hypergraphs (Pore-C).
It is biologically known that chromatin conformation is dynamically deformed and forms first-order and second-order topological structures to regulate gene expression in the cell differentiation process.
Reconstructing three-dimensional topological structures of DNA from these data is a crucial point of current cell research, and can be regarded as a problem of graph (or hypergraph) embeddings. Our demonstrations in this document make it clear that conventional data analysis using \texttt{Node2vec}-based methods are unlikely to capture chromatin conformation precisely due to the eradication of topology, and \emph{Topological Node2vec} hopes to resolve this.
The validation of Topological Node2vec as applied to epigenome data science is the primary concern of our future work.



\acks{This work was sponsored by JSPS Grant-in-Aid for Transformative Research Areas (A) (no. 22H05107 to Y. Hiraoka), JST MIRAI Program Grant (no. 22682401 to Y. Hiraoka),  JST PRESTO (no. JPMJPR2021 to Y. Imoto), AAP CNRS INS2I 2022 (to T. Lacombe), and JSPS Grant-in-Aid for Early-Career Scientists (no. 21K13822, to T. Yachimura), ASHBi Fusion Research Program (to Y. Imoto and K. Meehan).}

\bibliographystyle{abbrvnat}

\bibpunct{(}{)}{;}{a}{,}{,}

\bibliography{bibliography}


\newpage
\appendix

\section{Delayed Proofs}

\begin{proof}[Proof of \cref{prop_ce_gradient}] Let $v=v_i$, $1\leq j\leq m$, and $1\leq l\leq n$.
\begin{align*}
&\phantom{=}\frac{\partial H(T_v,C_v)}{\partial W_1(l,j)}=-\frac{\partial\sum_{1\leq a\leq n}T_v(a)u_v(a)}{\partial W_1(l,j)}+\frac{\partial\log\sum_{1\leq a\leq n} e^{u_v(a)}}{\partial W_1(l,j)}\\
&=-\sum_{1\leq a\leq n}T_v(a)\sum_{1\leq b\leq m}\dfrac{\partial W_1(i,b)W_2(b,a)}{\partial W_1(l,j)}+\sum_{1\leq a\leq n} e^{u_v(a)}\frac{\partial u_v(a)}{\partial W_1(l,j)}\cdot\dfrac{1}{\sum_{1\leq a\leq n} e^{u_v(a)}}\\
&=-\delta_{il}\sum_{1\leq a\leq n}(T_v(a)W_2(j,a)+W_2(j,a)C_v(a))\\
&
\\
&=\delta_{il}\sum_{1\leq a\leq n}W_2(j,a)(C_v(a)-T_v(a)).
\end{align*}
Next,
\begin{align*}
&\phantom{=}\frac{\partial H(T_v,C_v)}{\partial W_2(j,l)}=\frac{\partial(-\sum_{1\leq a\leq n}T_v(a)u_v(a) + \log\sum_{1\leq a\leq n} e^{u_v(a)})}{\partial W_2(j,l)}\\
&=-\frac{\partial\sum_{1\leq a\leq n}T_v(a)u_v(a)}{\partial W_2(j,l)}+\frac{\partial\log\sum_{1\leq a\leq n} e^{u_v(a)}}{\partial W_2(j,l)}\\
&=-\sum_{1\leq a\leq n}T_v(a)\frac{\partial u_v(a)}{\partial W_2(j,l)}+\sum_{1\leq a\leq n} \frac{\partial e^{u_v(a)}}{\partial W_2(j,l)}\cdot\dfrac{1}{\sum_{1\leq a\leq n} e^{u_v(a)}}\\
&=-\sum_{1\leq a\leq n}T_v(a)\sum_{1\leq b\leq m}\dfrac{\partial W_1(i,b)W_2(b,a)}{\partial W_2(j,l)}+\sum_{1\leq a\leq n} e^{u_v(a)}\frac{\partial u_v(a)}{\partial W_2(j,l)}\cdot\dfrac{1}{\sum_{1\leq a\leq n} e^{u_v(a)}}\\
&=-T_v(l)W_1(i,j)+W_1(i,j)\cdot\dfrac{e^{u_v(l)}}{\sum_{1\leq a\leq n} e^{u_v(a)}}\\
&=W_1(i,j)(C_v(l)-T_v(l)).
\end{align*}
That is, the gradient of $L_0(T_v,C_v)$ with respect to $W_1$ is non-zero only when taking partials with respect to the $i^\text{th}$ row of $W_1$ (recall $v=v_i$).
\end{proof}

\end{document}